\pdfoutput=1
\documentclass[conference]{IEEEtran}
\usepackage{times} 

\usepackage{amsmath}
\usepackage{amsfonts}
\usepackage{amssymb}
\usepackage{algorithmic}
\usepackage{graphicx}
\usepackage{graphics}
\usepackage{caption}
\DeclareCaptionType{copyrightbox}
\usepackage{subcaption}
\usepackage[usenames,dvipsnames]{color}
\usepackage{pbox}
\usepackage{pifont}
\usepackage{xspace}
\usepackage{xcolor}
\usepackage{array}
\usepackage{url}
\usepackage{hyperref}
\usepackage[boxed,vlined]{algorithm2e}
\usepackage{paralist} 
\usepackage{listings} 
\usepackage{tikz}
\usetikzlibrary{calc}

\newtheorem{problem}{Problem}

\newtheorem{lemma}{Lemma}
\newtheorem{theorem}{Theorem}

\newcommand{\hide}[1]{}

\newcommand{\bit}{\begin{compactitem}}
\newcommand{\eit}{\end{compactitem}}
\newcommand{\ben}{\begin{compactenum}}
\newcommand{\een}{\end{compactenum}}

\newcommand{\method}{\textsc{fBox}\xspace}
\newcommand{\spoken}{\textsc{SpokEn}\xspace}
\newcommand{\copycatch}{\textsc{CopyCatch}\xspace}
\newcommand{\oddball}{\textsc{OddBall}\xspace}
\newcommand{\blatant}{blatant\xspace}

\newcommand{\cross}{\textcolor{red}{\ding{56}}}
\newcommand{\tick}{\textcolor{OliveGreen}{\ding{52}}}

\begin{document}
\title{Spotting Suspicious Link Behavior with fBox: An Adversarial Perspective}

\author{\IEEEauthorblockN{Neil Shah}
\IEEEauthorblockA{Carnegie Mellon University\\
Pittsburgh, PA \\
neilshah@cs.cmu.edu}
\and
\IEEEauthorblockN{Alex Beutel}
\IEEEauthorblockA{Carnegie Mellon University\\
Pittsburgh, PA\\
abeutel@cs.cmu.edu}
\and
\IEEEauthorblockN{Brian Gallagher}
\IEEEauthorblockA{Lawrence Livermore Lab\\
Livermore, CA\\
bgallagher@llnl.gov}
\and
\IEEEauthorblockN{Christos Faloutsos}
\IEEEauthorblockA{Carnegie Mellon University\\
Pittsburgh, PA\\
christos@cs.cmu.edu}}


%


\maketitle

\begin{abstract}
How can we detect suspicious users in large online networks?
Online popularity of a user or product (via follows, page-likes, etc.)
can be monetized on the premise of higher ad click-through rates or increased sales.
Web services and social networks which incentivize popularity thus suffer from a major problem of fake connections from link fraudsters looking to make a quick buck.
Typical methods of catching this suspicious behavior
use spectral techniques to spot large
groups of often blatantly fraudulent (but sometimes honest) users.  
However, small-scale, stealthy attacks may go unnoticed due to the nature of low-rank eigenanalysis used in practice.

In this work, we take an adversarial approach to find and prove claims about the weaknesses of modern, state-of-the-art spectral methods and propose \method, an algorithm designed
to catch small-scale, \emph{stealth attacks} that slip below the radar.  Our algorithm has the following desirable properties:
(a) it has theoretical underpinnings,
(b) it is shown to be highly effective on real data
and
(c) it is scalable (linear on the input size).
We evaluate \method 
on a large, public 41.7 {\it million} node, 1.5
{\it billion} edge who-follows-whom social graph from Twitter in 2010 and with high precision 
identify many 
suspicious accounts which have persisted without suspension even to this day.

\end{abstract}


%
\IEEEpeerreviewmaketitle

\section{Introduction}
\label{sec:intro}
In an online network, 
how can we distinguish honest users from deceptive ones?
Since many online services rely on machine learning algorithms to recommend relevant content to their users, it is crucial to their performance that user feedback be legitimate and indicative of true interests.  
``Fake'' links via the use of sockpuppet/bot accounts can enable arbitrary (frequently spammy or malicious) users and products of varying nature seem credible and popular, thus degrading the online experience of users.
Unsurprisingly, numerous sites such as {\tt buy1000followers.co}, {\tt boostlikes.com} and {\tt buyamazonreviews.com} exist to provide services such as fake Twitter followers, Facebook page-likes and Amazon product reviews for typically just a few dollars per one-thousand fake links.

\hide{

    Because many online services rely on user feedback, it is crucial that the
    feedback be honest and legitimate.  However, this also
    creates a large market for attackers to falsely give credibility
    and popularity to certain people or content,
    e.g. through fake Twitter followers~\cite{stringhinitwitter12},
    LinkedIn connections, Facebook Likes \cite{copycatch13}, or eBay reviews
    \cite{netprobe07pandit}.  
    In fact, an online search yields numerous services to boost
    online popularity, often through dubious methods, with Twitter Followers or
    Page Likes sold by the thousands.
    
    In this work, we target detection of link fraud, characterized by attackers adding links, or connections, to customers in order to increase the customer's perceived popularity.  This type of fraudulent feedback is disruptive to many modern web services that
    use machine learning algorithms to recommend relevant
    content to their users.  For example, Netflix users are recommended movies
    based on what similar users enjoyed, Amazon users are
    suggested products based on previous purchases, and
    Facebook users are suggested Pages based on the interests of their
    social circles.  Falsely popularized content distorts
    recommendations and dishonestly engenders trust~\cite{fbsec12}. Given that
    link fraud is an unfortunate reality universal to incentivized online services, we aim to distinguish honest users from dishonest attackers and their customers by analysis on the input graph characterizing the links between users and objects.  This approach is the most direct in catching link fraud, given that such attacks can be inferred from edges in the input graph without the use of extraneous side information.  Although a number of previous efforts have attacked similar problems, it is poorly understood how these techniques can work together to combat link fraud and what limitations they suffer from.
    
}

Here we focus exactly on the link-fraud problem.  We take an adversarial approach to illustrate 
when and how current methods fail to detect fraudsters and design a new complementary algorithm, \method, to spot attackers who evade
these state-of-the-art techniques. 
Figure
\ref{fig:crown} 
showcases several suspicious accounts spotted by \method -- 
we elaborate on three of them, marked using the triangle, square and star glyphs.
All three are identified as outliers
in the \method Spectral Reconstruction Map (SRM) shown in Figure~\ref{fig:crown}b.
The corresponding Twitter profiles are shown in 
Figure \ref{fig:crown}c, and further manual inspection
shows that all three accounts exhibit suspicious behavior:
\bit
\item triangle: it has only 2 tweets but over 1000 followers
\item square: it is part of a 50-clique with suspicious names
\item star: it posts tweets advertising a link fraud service
\eit

Our main contributions are the following:

\begin{compactenum}
\item {\bf Theoretical analysis:} We prove limitations 
     of the detection range of spectral-based methods. 
 \item {\bf \method algorithm:} We introduce \method, 
        a {\em scalable} method
	 that {\em boxes-in} attackers, since 
	 it spots small-scale, stealth attacks  
	 which evade spectral methods.
   \item {\bf  Effectiveness on real data:} We apply \method to a real, 41.7
	   {\it million} node, 1.5 {\it billion} edge Twitter who-follows-whom
	   social graph from 2010 and identify many still-active accounts with suspicious follower/followee links, spammy
	   Tweets and otherwise strange behavior. 
\end{compactenum}

{\bf Reproducibility:}
 Our code is available at \url{http://www.cs.cmu.edu/~neilshah/code/}.  The Twitter dataset is also publicly available as cited in \cite{kwak10www}.

\begin{figure*}[!t]
\centering
\begin{subfigure}{.3\textwidth}
  \centering
  \includegraphics[width=0.98\textwidth]{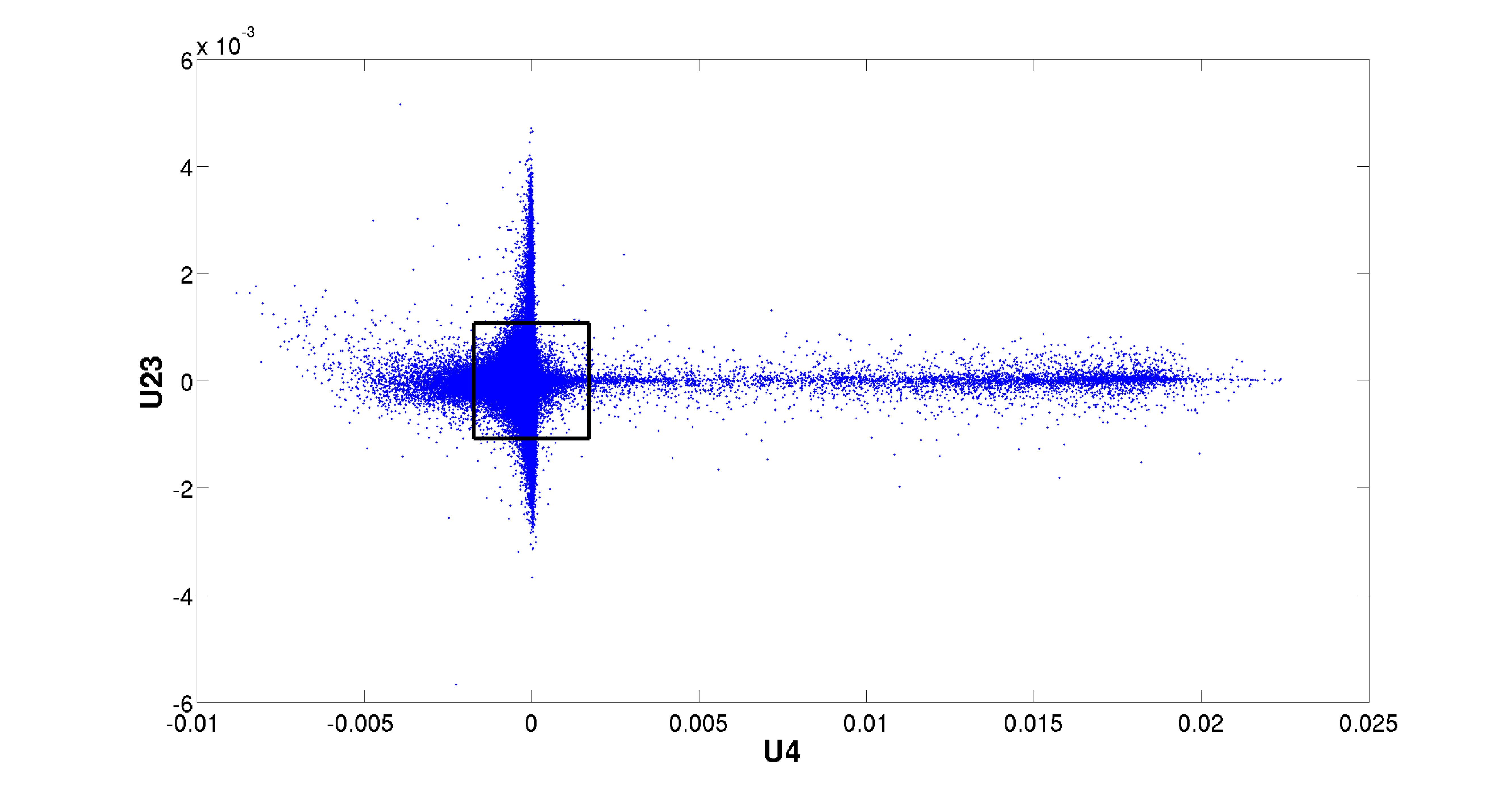}
  \vspace{10mm}
  \caption{Spectral subspace plot}
  \label{fig:crown_eigenplot}
\end{subfigure}%
\begin{subfigure}{.33\textwidth}
  \centering
  \includegraphics[width=0.97\textwidth]{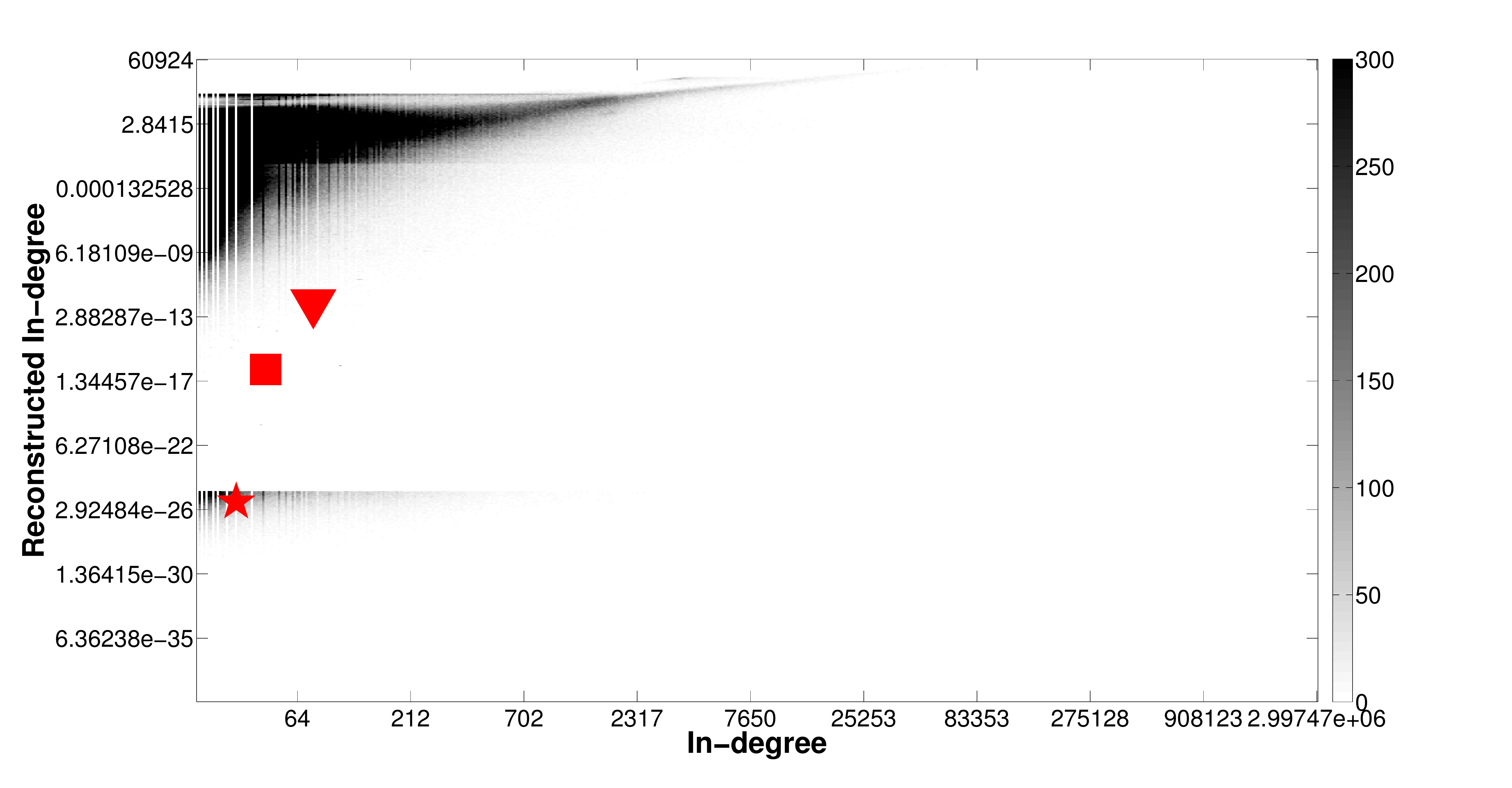}
  \vspace{8mm}
  \caption{Proposed ISRM}
  \label{fig:crown_isrm}
\end{subfigure}%
\begin{subfigure}{.3\textwidth}
  \centering
  \includegraphics[width=0.77\textwidth]{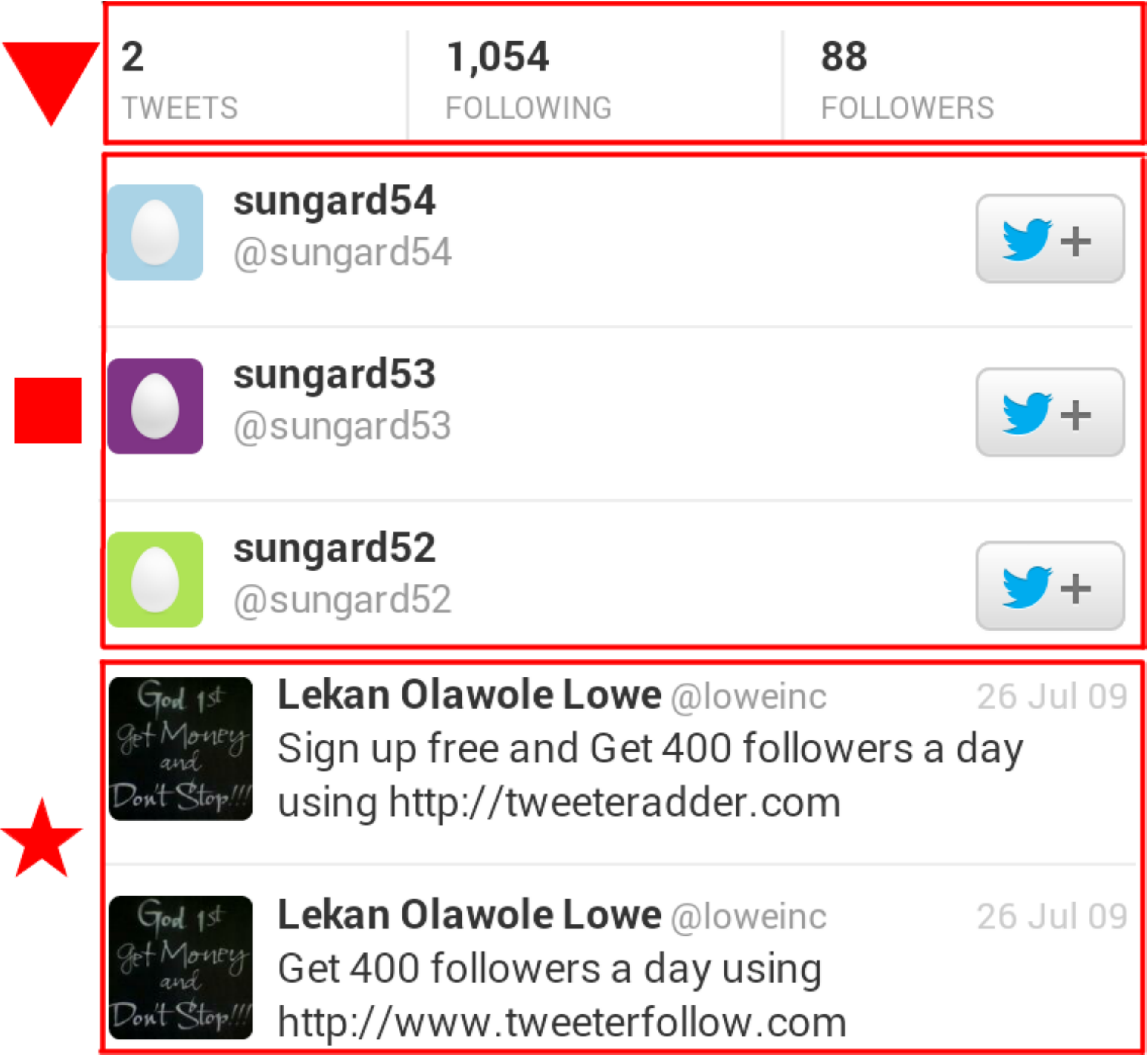}
  \caption{Suspicious accounts}
  \label{fig:crown_fraudsters}
\end{subfigure}
\caption{\method catches stealth attacks which are missed by spectral methods.  (a) shows a spectral subspace plots on the Twitter social graph which identifies \blatant attacks but ignores stealth attackers (at the origin).  (b) portrays how the proposed \method ISRM (In-link Spectral Reconstruction Map) can describe these users by their \emph{reconstruction degree} and identifies several with improbably poor reconstruction.  (c) shows their suspicious profiles with matching glyphs (see text for details).}
\label{fig:crown}
\end{figure*}


\section{Background and Related Work}
\label{sec:background}
We begin by reviewing in detail several of the current state-of-the-art methods
in web fraud and spam detection.  Table \ref{tbl:qualcomp} shows a qualitative comparison between various link fraud detection methods.

\subsection{Spectral methods}

\makeatletter
\newcommand{\thickhline}{%
    \noalign {\ifnum 0=`}\fi \hrule height 1pt
    \futurelet \reserved@a \@xhline
}
\newcolumntype{"}{@{\hskip\tabcolsep\vrule width 1.5pt\hskip\tabcolsep}}
\makeatother

\begin{table*}
\centering
\caption{Qualitative comparison between \method and other link fraud detection methods.}
\label{tbl:qualcomp}
\begin{tabular}{|l||c|c|c|c"c|c|}
\hline 
\multicolumn{1}{ |c|| }{\bf P r o p e r t i e s}  & \multicolumn{5}{c|}{\bf M e t h o d s} \\
\textbf{}  & {\bf \spoken} & {\bf Spectral Subspace Plotting} & {\bf \copycatch} & {\bf \oddball} & {\bf \method} \\
\hline \hline
\textbf{Detects stealth attacks} & \cross & \cross & \cross & \cross & \tick  \\ \hline
\textbf{Camouflage resistant} & \cross & \tick & \tick & \cross & \tick \\ \hline
\textbf{Offers visualization} & \tick & \tick & \cross & \cross & \tick \\ \hline
\end{tabular}
\end{table*}

We classify techniques that analyze the latent factors produced in graph-based eigenanalysis or matrix/tensor decomposition as spectral methods.  These algorithms seek to find patterns in the graph decompositions to extract coherent groups of users or objects.  Prakash et al's work on the EigenSpokes pattern~\cite{eigenspokes10} and Jiang et al's work on spectral subspaces of social networks~\cite{jiang14} are two such approaches that we will primarily focus on and which have been employed on real datasets to detect suspicious link behavior. \cite{ying2011spectrum} uses a similar analysis of spectral patterns, but focuses on random link attacks (RLAs), which have different properties than link fraud and therefore produce different patterns.

These works utilize the Singular Value Decomposition (SVD) of the input graph's adjacency matrix in order to group similar users and objects based on their projections.  Recall that the SVD of a $u \times o$ matrix {\bf A} is defined as ${\bf A = U{\Sigma}V^T}$, where {\bf U} and {\bf V} are $u \times u$ and $o \times o$ matrices respectively containing the left and right singular vectors, and ${\bf \Sigma}$ is a $u \times o$ diagonal matrix containing the singular values of {\bf A}.  Both papers note the presence of unusual patterns (axis-aligned spokes, tilting rays, pearl-like clusters, etc.) when plotting the  singular vectors ${\bf U_{i}}$ and ${\bf U_{j}}$ for some $i, j \leq k$, where $k$ is the SVD decomposition rank, indicative of suspicious lockstep behavior between similar users.  The authors use these patterns to chip out communities of similar users from input graphs. 

Beyond directly searching for suspicious behavior, spectral methods have been used for a variety of applications.
\cite{malspot14} builds off the above work to use tensor decomposition for network intrusion detection. 
\cite{cobafi} proposes a robust collaborative filtering model that clusters latent parameters to limit the impact of fraudulent ratings from potential adversaries.
\cite{ng2001spectralA} and \cite{huang2008spectral} propose using eigenvectors of graph decompositions for graph partitioning and community detection.


Although spectral methods have shown promise in finding large communities and blatantly suspicious behavior in online networks, they are universally vulnerable given knowledge of the decomposition rank $k$ used in a given implementation.  All techniques operating on large graphs use such a parameter in practical implementations given that matrix decompositions are very computationally expensive~\cite{kang2014heigen}.  Previous spectral methods have generally chosen small values of $k < 100$ for purposes of computability.  As we will show in Section~\ref{sec:adversarial}, knowledge of  $k$ or the associated singular value threshold (inferrable from sample datasets online) enables an intelligent adversary to engineer attacks to fall below the detection threshold.  

\subsection{Graph-traversal based methods}

A wide variety of algorithms have been proposed to directly traverse the graph to find or stop suspicious behavior.
\cite{shrivastava2008mining}  offers a random walk algorithm for detecting RLAs.
\cite{ghosh2012understanding} proposes a PageRank-like approach for penalizing promiscuous users on Twitter, but is unfortunately only shown to be effective in detecting already caught spammers rather than detecting new ones. 
\cite{netprobe07pandit} uses belief propagation to find near-bipartite cores of attackers on eBay. 

However, most similar in application is Beutel et al's \copycatch algorithm to find suspicious lockstep behavior in Facebook Page Likes~\cite{copycatch13}. 
\copycatch is a clustering method that seeks to find densely connected groups in noisy data through restricted graph traversal, motivated with the intuition of fraud taking the form of na\"{i}vely created bipartite cores in the input graph. The algorithm uses local search in the graph to find dense temporally-coherent near-bipartite cores (TNBCs) given attack size, time window and link density parameters.

Clustering methods like \copycatch are able to avoid detection problems caused by camouflage (connections created by attackers to legitimate pages or people for the purposes of appearing like honest users) given that they ignore such links if the attacker is party to any TNBC.  However, identifying the appropriate ``minimal attack'' parameters is nontrivial.  Non-conservative parameter estimates will result in many uncaught attackers whereas excessively conservative estimates will result in numerous false positives.  From an adversarial point-of-view,  we argue that the cost of incurring false positives and troubling honest users is likely not worth the added benefit of catching an increased number of attackers after some point.  
Therefore, an alternative approach to catch stealth attacks falling below chosen thresholds is necessary.

\subsection{Feature-based methods} 
Spam and fraud detection has classically been framed as a feature-based classification problem, e.g. based on the words in spam email or URLs in tweets.  
However, \cite{berkeleyspam} focuses on malicious Tweets and finds that blacklisting approaches are too slow to stem the spread of Twitter spam.  \oddball~\cite{akoglu2010oddball} proposes features based on egonets to find anomalous users on weighted graphs.  
\cite{domingoskdd04} and \cite{lowd2005adversarial} take a game theoretic
approach to learning simple classifiers over generic features to detect spam.
While related in the adversarial perspective, these approaches focus on general
feature-based classification as used for spam email, rather than
graph analysis as is needed for link fraud detection.

\section{An Adversarial Analysis - Our Perspective}
\label{sec:adversarial}
In this section, we examine the exploitability of state-of-the-art methods
from an adversarial point-of-view and present lemmas and theorems detailing the limitations of these methods.  Particularly, we demonstrate through
theoretical analysis that existing methods are highly vulnerable to evasion by
intelligent attackers.  Table~\ref{tbl:symb} contains a comprehensive list of
symbols and corresponding definitions used in our paper.

\begin{figure}[!t]
\centering
\begin{subfigure}{.16\textwidth}
  \centering
  \includegraphics[width=0.98\textwidth]{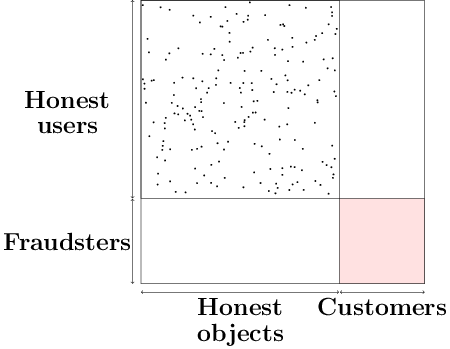}
  \caption{Na\"{i}ve}
  \label{fig:block_attack}
\end{subfigure}%
\begin{subfigure}{.16\textwidth}
  \centering
  \includegraphics[width=0.98\textwidth]{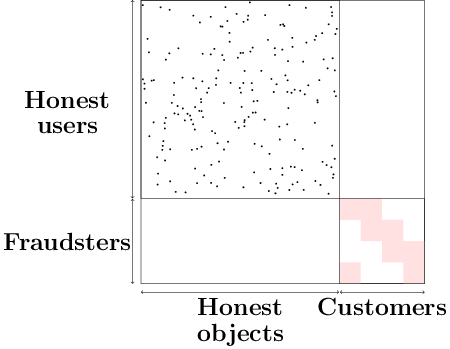}
  \caption{Staircase}
  \label{fig:staircase_attack}
\end{subfigure}%
\begin{subfigure}{.16\textwidth}
  \centering
  \includegraphics[width=0.98\textwidth]{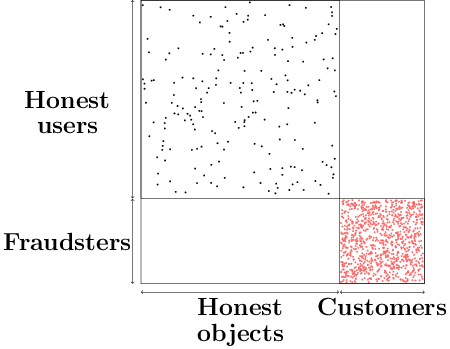}
  \caption{Random graph}
  \label{fig:random_attack}
\end{subfigure}%
\caption{(a), (b) and (c) show the different types of adversarial attacks we characterize.}
\label{fig:attack_types}
\end{figure}

\begin{table*}
\centering
\caption{Frequently used symbols and definitions}
\label{tbl:symb}
    \begin{tabular}{|l|l|}
    \hline
    {\bf Symbol}               & {\bf Definition}                                                                                                                         \\ \hline
    $u$ and $o$              & Number of user and object nodes described by the input graph                                                                       \\ \hline
    $\mathcal{U}$ and $\mathcal{P}$              & Sets of indexed rows and columns corresponding to user and object nodes in the input graph                                         \\ \hline
    {\bf A}                    & $u \times o$ input graph adjacency matrix where ${\bf A}_{x,y} = 1$ if a link exists between user node $x$ and object node $y$                           \\ \hline
    $f$ and $c$              & Number of attacker and customer nodes described by the attack graph                                                               \\ \hline
    $s$                    & Number of fraudulent actions each customer node has paid commission for in the attack graph                                        \\ \hline
    $\mathcal{F}$ and $\mathcal{C}$              & Sets of indexed rows and columns corresponding to attacker nodes and customer nodes in the attack graph                           \\ \hline
    {\bf S}                    & $f \times c$ attack graph adjacency matrix where ${\bf S}_{x,y} = 1$ if a link exists between attacker node $x$ and customer node $y$                   \\ \hline
    $k$                    & Decomposition rank parameter used by spectral methods                                                                              \\ \hline
    $\lambda_k$ and $\sigma_k$ & $k$th largest eigenvalue and singular value of a given matrix (largest values for $k = 1$) \\ \hline
    $m$, $n$ and $p$           & Bipartite core size and edge probability parameters used by clustering methods                                                     \\ \hline
    \end{tabular}

\end{table*}

Given knowledge of the detection threshold used by a certain service, how can an attacker engineer smart attacks on that service to avoid detection by fraud detection
methods? 

Formally, we pose the following adversarial problem:
\begin{problem}
{\bf Given} an input graph adjacency matrix {\bf A}, with rows and columns corresponding to users and objects, {\bf engineer} a stealth attack which falls \emph{just below} the minimum sized attack detectable by modern state-of-the-art fraud detection methods.
\label{prob1}
\end{problem} 
As previously described, most detection methods focus on finding fairly blatant 
bipartite cores or cliques in the input graph.  Therefore, if an adversary knows the
minimum size attack that detection methods will catch, he can carefully engineer attacks to fall
just below that threshold.
For clustering approaches like \copycatch, this threshold is clearly set based
on input parameters, and the attacker can simply use fewer accounts than specifiedto avoid detection.

However, for spectral methods like \spoken, the possible attack size for an
adversary is unclear.
We argue that from an adversarial perspective, these spectral methods have a
detection threshold based on the input graph's singular values.
For a rank $k$ SVD used in these methods, this threshold is governed by the
$k$th largest singular value, $\sigma_k$.
In practice, an adversary could estimate $\sigma_k$
from the results of
various experimental attacks conducted at distinct scales, or by conducting analysis on publicly available social network data.  
Once an adversary has such an estimate, we show that it
is easy to conduct attacks on the graph adjacency matrix {\bf A} that will
necessarily lie below this threshold and avoid detection.

To analyze what type of attacks can evade detection by spectral methods,  
let us consider that there are $c$ customers who have each commissioned an attacker with $f$ nodes in his botnet for $s$ fraudulent actions (page likes,
followers, etc.), where $s \leq f$.
This type of attack can be considered as an injected submatrix {\bf S} of size
$f \times c$, where rows correspond to attacker nodes (controlled by a single
fraudulent operator) in the set of users ($\mathcal{F} \subset \mathcal{U}$)
and the columns represent customers in the set of objects ($\mathcal{C} \subset
\mathcal{P}$).  
In this formulation, the desired in-degree of all nodes in {\bf S} is $s$.

As described earlier, an attack will only be detected by a spectral algorithm
if it appears in the top $k$ singular values/vectors.  Therefore, our goal as an adversary becomes to understand the spectral properties of our attacks and ensure that they
do not project in the rank $k$ decomposition.
We can consider the spectral properties of {\bf S} in isolation
from the rest of the graph, as it is well known that the spectrum of a
disconnected graph is the union of the spectra of its connected components.
From this, we deduce that it is sufficient to consider only the
representation of {\bf S} and ignore the remainder of {\bf A} when trying to
minimize the leading singular value that the attack contributes to the singular
spectrum of {\bf A}.  {\it Therefore, an attack {\bf S} with leading singular
value $\sigma'$ will go undetected by spectral methods if $\sigma' < \sigma_k$, where $\sigma_k$ is
the $k$th largest singular value computed for the adjacency matrix {\bf A}.}

Having reduced the problem of adversarial injection to distributing some amount
of fraudulent activity over the $f \times c$ matrix {\bf S}, we next consider
several distinct patterns of attack which characterize types of fraudulent
behavior discovered in the analysis of prior work.  Specifically, we explore
three fraud distribution techniques: na\"{i}ve, staircase and random
graph injections.  Figure \ref{fig:attack_types} gives a pictorial representation of each of these types of attacks.
We evaluate the suitability of each attack for an adversary on the basis of the
leading singular value that the pattern generates.

\subsection{Na\"{i}ve Injection}
\label{sec:naivemethod}
This is the most notable attack pattern considered in prior work. The na\"{i}ve
injection distributes the $sc$ total fraudulent actions into an $s \times c$
submatrix of {\bf S}.  Thus, only $s$ of the $f$ attacker nodes perform any
fraudulent actions, and all fraudulent actions are distributed between these
$s$ nodes. In graph terms, this is equivalent to introducing a $s \times
c$ complete bipartite core.  Such an attack corresponds to attackers
na\"{i}vely linking the same set of $s$ nodes to each of the $c$
customers, producing a full block in {\bf A}.  Figure~\ref{fig:block_attack}
shows a visual representation of such an attack.

\begin{lemma}
	The leading singular value of an $s \times c$ bipartite core injection is $\sigma_1 = \sqrt{cs}$.
\label{lemma:naive}
\end{lemma}

\begin{proof}
	Since {\bf S} is a full block, where ${\bf S_{i,j}} = 1$ for all $i \leq s$, $j \leq c$, ${\bf SS^T}$ must be an $s \times s$ matrix where ${\bf {SS^T}_{i,j}} = c$ for all $i,j \leq s$.  By the Perron-Frobenius Theorem for non-negative matrices, the leading eigenvalue of ${\bf SS^T}$ $\lambda_1$ is bounded by
\begin{displaymath}
\operatorname*{min}_i \sum\limits_{j} {\bf {SS^T}_{i,j}} \leq \lambda_1 \leq \operatorname*{max}_i \sum\limits_{j} {\bf {SS^T}_{i,j}} \textrm{ for } i \leq s
\end{displaymath}
Given that the row sums are equal to $cs$, $\lambda_1 = cs$ for ${\bf SS^T}$.  Since the singular values of {\bf S} are equal to the square roots of the eigenvalues of ${\bf SS^T}$ by definition, it follows naturally that the leading singular value is $\sigma_1 = \sqrt{cs}$ for {\bf S}.     
\end{proof}

\subsection{Staircase Injection}
\label{sec:steppingmethod}
The staircase injection (discovered in \cite{jiang14}) evenly distributes $cs$ fraudulent actions over $f$ attacker nodes.  However, unlike in
the na\"{i}ve method, where each node that performs any fraudulent actions
does so for each of the $c$ customers, the staircase method forces different
subsets of nodes to associate with different subsets of customers.  This
distribution results in the {\bf S} matrix looking like a staircase of links.
Figure~\ref{fig:staircase_attack} shows a visual representation of such an
attack.

We restrict our analysis here to staircase injections in which all users have
equal out degrees $o$ and equal in degrees $i$, though $o$ need not equal $i$.
When out degrees and in degrees are not equal, users and objects do not have
uniform connectivity properties which complicates calculations.  In particular,
we assume that the periodicity of the staircase pattern, given by $t =
lcm(s,f)/s$ is such that $t|c$ to ensure this criteria.  However, for large
values of $c/t$, $\sigma_1 \approx{s\sqrt{c/f}}$ given LLN.

\begin{theorem}
	The leading singular value of an $s, c, f$ staircase injection is $\sigma_1 = s\sqrt{c/f}$. 
\label{thm:staircase}
\end{theorem}

\begin{proof}
The staircase injection is equivalent to a random graph-injection of $f \times c$ with edge probability $p = s/f$.  The reduction is omitted due to space constraints.  Such a random graph injection has leading singular value $s\sqrt{c/f}$ (proof given in Section~\ref{sec:randommethod}).
\end{proof}

\subsection{Random Graph Injection}
\label{sec:randommethod}
The random graph injection bears close resemblance to the near-bipartite core
with density $p$ attack noted in \cite{copycatch13}.  The random graph
injection distributes $\approx{sc}$ fraudulent actions over the $f$ attacker nodes
approximately evenly.  Figure~\ref{fig:random_attack} shows a visual
representation of such an attack.  This approach assigns each node a fixed
probability $p = sc/cf = s/f$ of performing a fraudulent operation associated
with one of the $c$ customers.  Given LLN, the average number of fraudulent
operations per customer will be close to the expected value of $s$, and as a
result the total number of fraudulent actions will be close to $sc$.  The
random graph injection is similar to the Erd\"{o}s-R\'{e}nyi model defined by
$G(n,p)$~\cite{erdos1959}, except we consider a directed graph
scenario with $cf$ possible edges.  However, as Erd\"{o}s and R\'{e}nyi studied
the asymptotic behavior of random graphs, their results are applicable here as
well.  

\begin{theorem}
	The leading singular value of an $s, c, f$ directed random bipartite graph is $\sigma_1 \sim s\sqrt{c/f}$. 
\label{thm:random}
\end{theorem}

\begin{proof}
Given that probability of an edge between an attacker node and a customer is $p = s/f$, it is apparent that 
\begin{displaymath}
E({\bf {SS^T}_{i,j}}) = {p^2}c \textrm{ for } i,j \leq f
\end{displaymath}
since the value of each cell in the $f \times f$ matrix ${\bf SS^T}$ will be a result of the inner product of the corresponding row and column vectors of length $c$ with probability $p$ of a nonzero entry at any $i \leq c$.  Since each row in ${\bf SS^T}$ has $f$ entries,
\begin{displaymath}
E(\sum \limits_{j} {\bf {SS^T}_{i,j}}) = {p^2}cf \textrm{ for } i \leq f
\end{displaymath}
By the Perron-Frobenius theorem for non-negative matrices, the leading eigenvalue $\lambda_1$ of ${\bf SS^T}$ will be bounded by
\begin{displaymath}
\operatorname*{min}_i \sum\limits_{j} {\bf {SS^T}_{i,j}} \leq \lambda_1 \leq \operatorname*{max}_i \sum\limits_{j} {\bf {SS^T}_{i,j}} \textrm{ for } i \leq f
\end{displaymath}
Given that the row sums are all approximately equal to ${p^2}cf = c{s^2}/f$ (exactly equal to $c{s^2}/f$ if edges in {\bf S} are perfectly uniformly distributed), the leading eigenvalue is $\lambda_1 \sim c{s^2}/f$ for ${\bf SS^T}$.  Since the singular values of {\bf S} are equal to the square roots of the eigenvalues of ${\bf SS^T}$, it follows naturally that the leading singular value is $\sigma_1 = s\sqrt{c/f}$ for {\bf S}.
\end{proof}

Note that the staircase injection discussed earlier in this section is exactly the case of edges in {\bf S} being perfectly uniformly distributed.  For this reason, the leading eigenvalue of ${\bf SS^T}$, where {\bf S} is an $s, c, f$ staircase injection is $\lambda_1 = c{s^2}/f$, and thus the leading singular value is $\sigma_1 = s\sqrt{c/f}$ for {\bf S}.

\subsection{Implications and Empirical Analysis} 


\begin{figure*}[htbp!]
\centering
\begin{subfigure}{.24\textwidth}
  \centering
  \includegraphics[width=0.95\textwidth]{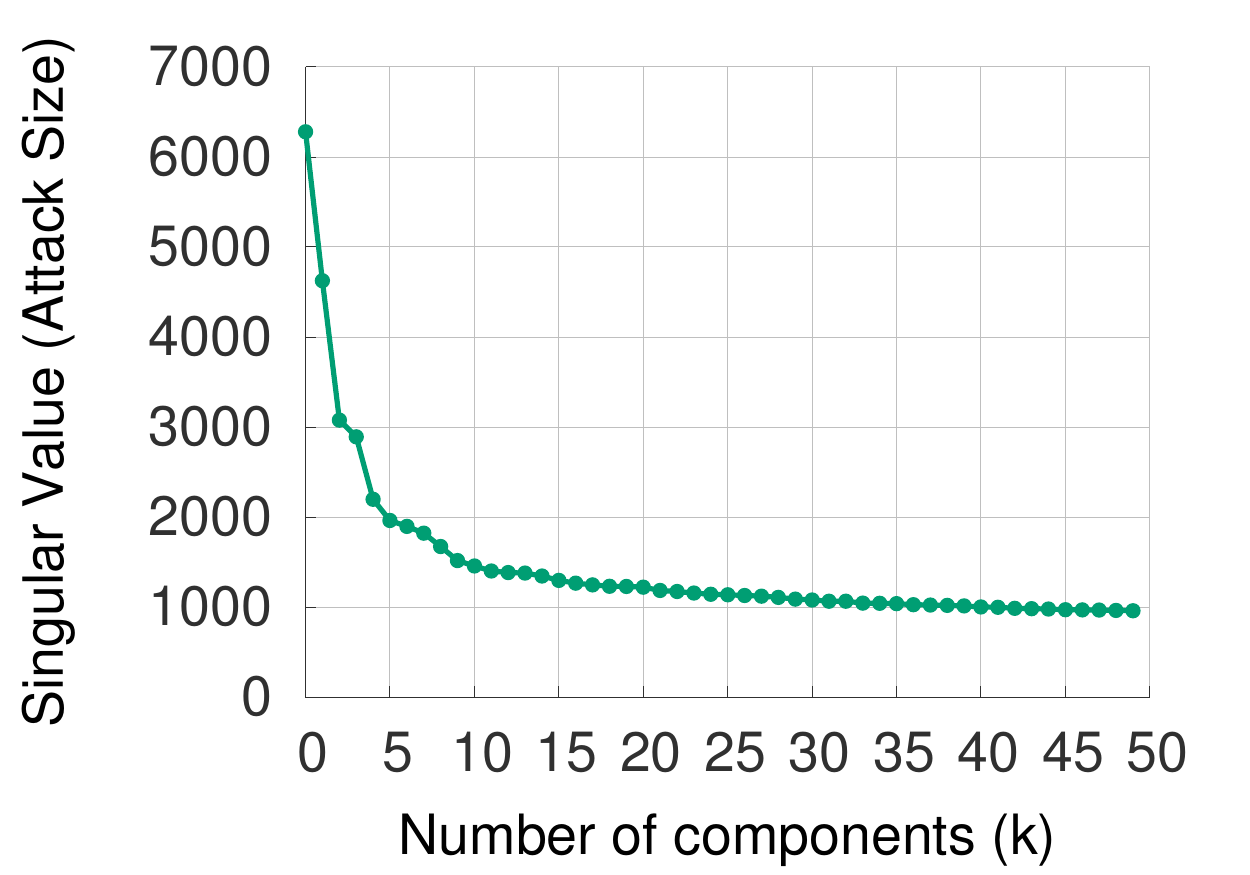}
  \caption{Twitter Followers}
  \label{fig:twitter_svd}
\end{subfigure}%
\begin{subfigure}{.24\textwidth}
  \centering
  \includegraphics[width=0.95\textwidth]{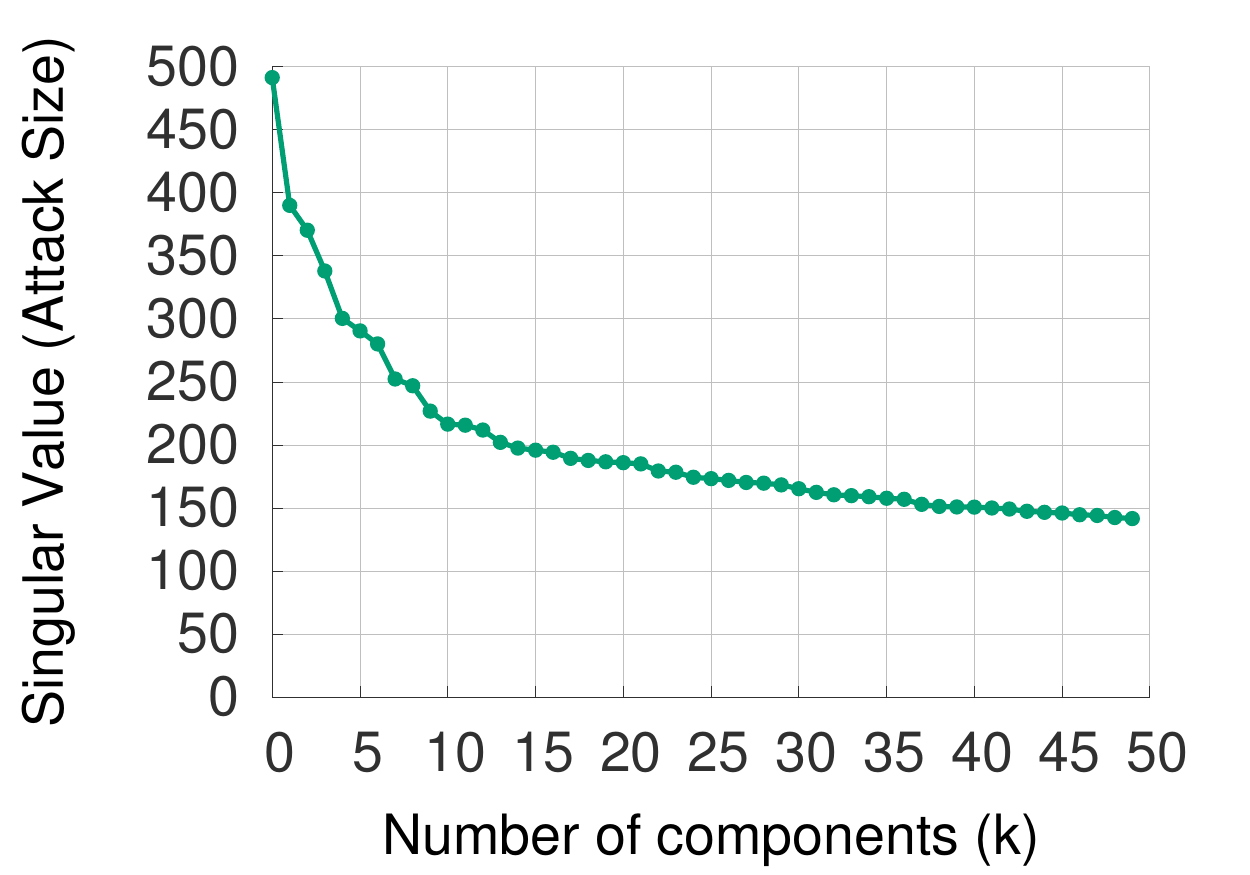}
  \caption{Amazon Reviews}
  \label{fig:amazon}
\end{subfigure}%
\begin{subfigure}{.24\textwidth}
  \centering
  \includegraphics[width=0.95\textwidth]{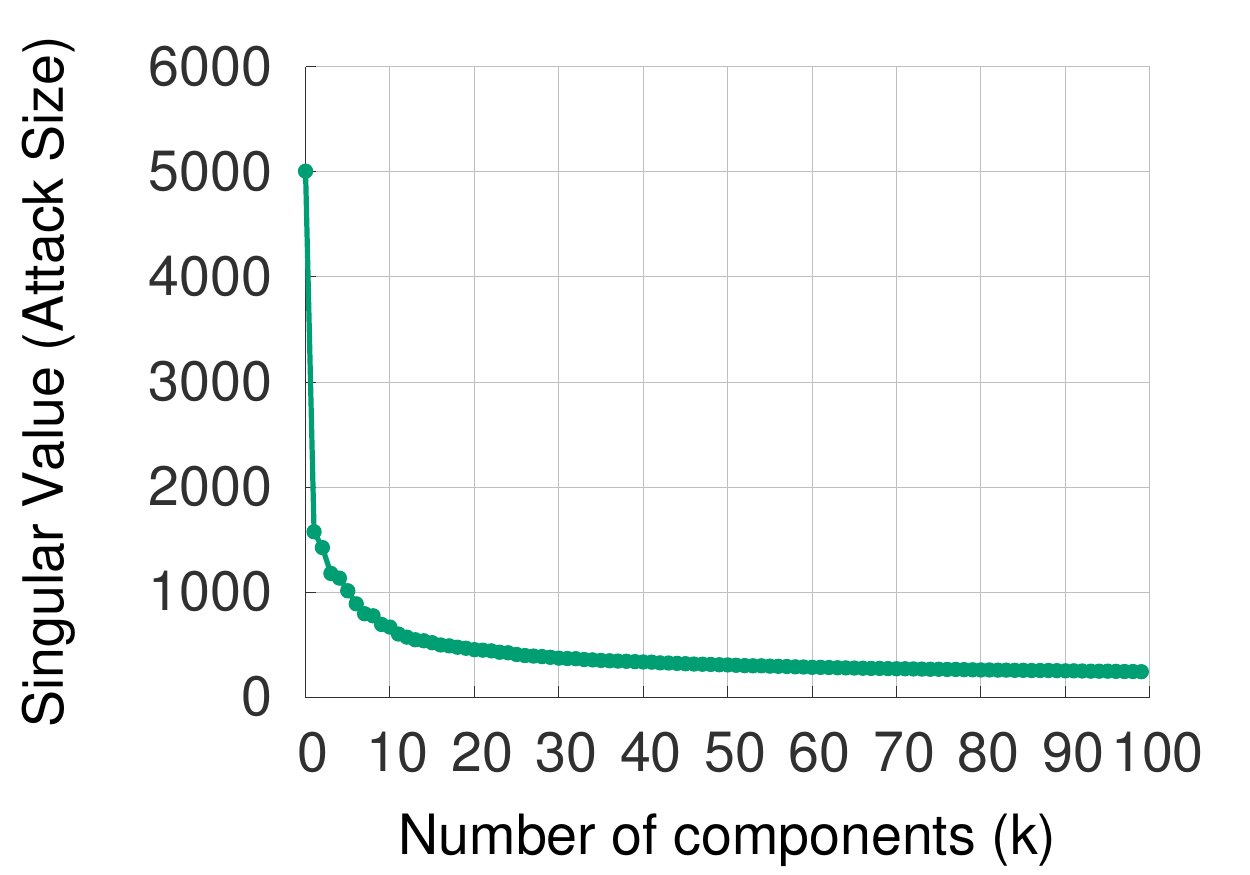}
  \caption{Netflix Ratings}
  \label{fig:netflix}
\end{subfigure}
\begin{subfigure}{.24\textwidth}
  \centering
  \includegraphics[width=0.95\textwidth]{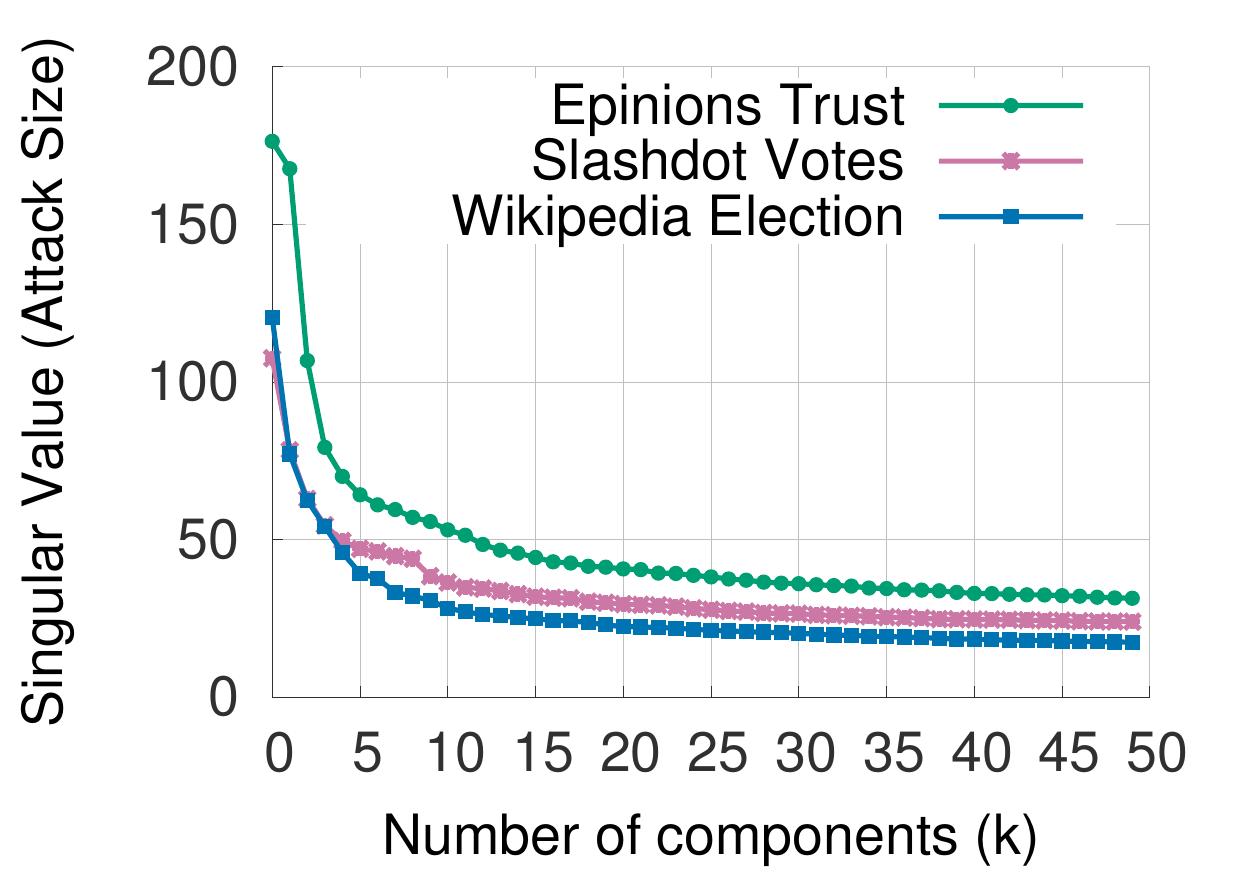}
  \caption{Smaller Graphs}
  \label{fig:other_graphs}
\end{subfigure}
\caption{Skewed singular value distribution in real networks --- spectral ($k$-rank SVD) approaches suffer from stealth attacks.   (a), (b), (c) and (d) show distributions for corresponding networks which allow stealth attacks capable of signficantly impacting local network structure to go undetected.}
\label{fig:svd_graphs}
\end{figure*}

Thus far, we have discussed three different types of potential attack patterns
for a fixed number of fradulent actions and theoretically derived expressions
concerning the leading singular value that they contribute to the singular
spectrum of {\bf A}.  Two of the attack
patterns, the staircase and random graph injections, produce leading singular
values $\sigma_1$ of exactly and approximately $s\sqrt{c/f}$ respectively.  Conversely, na\"{i}ve injection results in a leading singular
value of $\sigma_1 = \sqrt{cs}$.  Given these results, it is apparent that na\"{i}ve injection is the least suitable for an adversarial use, since it will
necessarily produce a larger singular value than the other two methods given that $s \leq f$.  This result is intuitive: the leading eigenvalue of a matrix is a measure of effective connectivity, and packing fraudulent actions into a full block matrix results in higher
connectivity than spreading the actions out over a large, sparse matrix. Our results beget two important conclusions:
\begin{compactenum}
\item Fraud detection tools must consider modes of attack other than na\"{i}ve
	injection --- more intelligent and less detectable means of attack exist
	and are being used.
\item Given knowledge of the effective singular value threshold $\sigma_k$ used
	by spectral detection methods, or $m$, $n$, $p$ parameter choice for
	clustering based methods, attackers can easily engineer attacks of scale up
	to \emph{just below} the threshold without consequence.
\end{compactenum}

\begin{table}[tbp!]
	\centering
	\caption{Graphs used for empirical analysis}
	\label{tab:datasets}
	\begin{tabular}{|c|c|c|}
		\hline
		{\bf Graph} & {\bf Nodes} & {\bf Edges}  \\ \hline \hline
		Twitter~\cite{kwak10www} & 41.7 million & 1.5 billion   \\ \hline
		Amazon~\cite{mcauley2013hidden} & 6m users \& 2m products & 29 million\\ \hline
		Netflix~\cite{netflix}  & 480k users \& 17k videos  & 99 million\\ \hline
		Epinions~\cite{leskovec2010signed} &  131,828 & 841,372\\ \hline
		Slashdot~\cite{leskovec2010signed} & 82,144 & 549,202 \\ \hline
		Wikipedia~\cite{leskovec2010signed} & 8274 & 114,040 \\ \hline
	\end{tabular}
\end{table}

To demonstrate that this leaves a significant opening for attackers, we
analyze the distribution of singular values for a variety of real world graphs
and show just how easy it is to construct attacks which slip below the radar.
In particular, we compute the SVD for six different real world graphs:
Twitter's who-follows-whom social graph, Amazon's bipartite graph of user
reviews for products, Netflix's graph of user reviews for movies, Epinions's
network of who-trusts-whom, Slashdot's friends/foe social graph, and
Wikipedia's bipartite graph of votes for administrators.  For each graph,
we turn it into a binary bipartite graph and compute the SVD for a fixed rank.
The properties of the datasets can be seen in Table \ref{tab:datasets} and the
results can be seen in Figure \ref{fig:svd_graphs}.

In Figure
\ref{fig:twitter_svd} we observe the top $k = 50$ singular values for the Twitter graph.  We
see that the largest singular value is over 6000, but as $k$ increases the
singular values begin to settle around 1000, with $\sigma_{50} = 960.1$.  Lemma~\ref{lemma:naive}
implies that an attacker controlling 960 accounts could use them to follow 960
other accounts and avoid projecting onto any of the top 50 singular vectors.
Note that Lemma~\ref{lemma:naive} also implies that an attacker could add 92
thousand followers to 10 lucky accounts and also go undetected.  These are very large numbers of followers that could
significantly shift the perception of popularity or legitimacy of accounts.
Common spectral approaches would fail to detect such attacks.
                                                                   
A similar analysis can be made for the other graphs.
Figure \ref{fig:amazon} shows that $\sigma_{50} = 141.6$ in the Amazon review graph.  Therefore, attackers could add 140 reviews for 140 products without projecting onto the top 50 singular vectors.  Considering the
average product has 12.5 reviews and a product in the $99^{\rm th}$ percentile
has 187 reviews, 140 reviews is sufficiently large to sway
perception of a product on Amazon.


As seen in Figure \ref{fig:netflix}, we find that $\sigma_{50} = 309.7$ and
$\sigma_{100} = 243.4$ for the Netflix ratings graph.
Therefore, attackers could na\"{i}vely add an injection of 240
ratings to 240 videos from 240 accounts and avoid detection in the top 100
singular vectors.

For the Epinions network, we see in Figure \ref{fig:other_graphs} that
$\sigma_{50} = 31.4$.  Although this value is much smaller than that for other
graphs, the Epinions network is small and sparse, with the average user having
an in-degree of 6.4.  Based on this singular value, an attacker adding 30 edges
(statements of trust or distrust) to 30 users would significantly influence the
external view of those users.

In Slashdot's friend vs. foe graph, $\sigma_{50} = 23.9$, as seen in Figure \ref{fig:other_graphs}.
This means that attackers
could add 23 ratings for 23 users while avoiding spectral detection.
Considering that the average in-degree for accounts in this network is 6.7,
adding 23 edges would significantly impact the perception of a user.

Lastly, we examine the graph of 2794 administrative elections on Wikipedia.
As shown in Figure
\ref{fig:other_graphs}, $\sigma_{50} = 17.5$.  This implies that 17
users could for 17 elections all vote together and avoid detection.
In fact, 31\% of elections were settled by 17 votes or less. An attacker could also modify the shape of the attack such that 5 users would each
receive 57 votes, enough to win 72\%.
Given an attack of this scale, a small group of accounts could cooperate to unfairly rig election outcomes. 

From these examples across a variety of networks, we see that
using spectral approaches for catching fraud leaves a wide opening for
attackers to manipulate online graphs.


\section{Proposed Framework for Fraud Behaviors}
\label{sec:bigpic}
As demonstrated in Section~\ref{sec:adversarial}, current detection methods are
effective in catching blatant attacks, but drop in efficacy as the attack
size decreases.  Though the scale of attacks detected is defined differently
for various datasets given distinct decomposition rank $k$, such a
detectability cross-over point necessarily exists given the well-defined nature
of the singular value produced by common types of attacks.  In this section, we
give a broader overview of possible attack modes and the capabilities of
current methods in dealing with them.  Table \ref{tab:bigpic} illustrates how
current techniques fit into our classification of suitable defenses against
four different attack types and how the proposed \method algorithm can fill in
the remaining holes to provide a more holistic framework for fraud detection.

\begin{table}
	\centering
	\caption{Types of attacks and suitable detection techniques}
	\label{tab:bigpic}
	\bgroup
	\def\arraystretch{2}
	\begin{tabular}{c|c|c|}
	   & {\bf No Camouflage } & {\bf Camouflage }\\ \hline
	   {\bf Blatant Attacks }& \pbox{10cm}{\spoken; \\ \copycatch} & \pbox{10cm}{Spectral subspace plotting; \\ \copycatch} \\ \hline
	   {\bf Stealth Attacks }& (proposed) {\bf \method} & (proposed) {\bf \method} \\ \hline
	\end{tabular}
	\egroup
	
\end{table}

The four types of attacks we broach in this work are classified based on two dichotomies --- the \emph{scale} of attack and the presence of \emph{camouflage}.  The scale of attack concerns whether an attack of some size defined in terms of the aformentioned $s$, $c$ and $f$ parameters in the context of a given dataset (and decomposition rank $k$ for spectral methods), is detectable or not.  The attack could be staged using any of the fraud distribution patterns discussed in Section~\ref{sec:adversarial}.  In the context of clustering methods, scale is more formally defined by the minimal attack size parameters used.  Camouflage refers to uncommissioned actions conducted by attackers in order to appear more like honest users in the hopes of avoiding detection.  For example, camouflage on Twitter is most commonly seen as attackers following some honest users for free in addition to paid customers.  Attacks with camouflage are more difficult to detect than those without, given the increased likelihood of a practitioner to overlook suspicious actions. 

\subsection{Blatant Attack/No Camouflage}

Of the four types, \blatant attacks without camouflage are the easiest to
spot.  Blatant attacks whose singular values are above the threshold $\sigma_k$
and thus appear in the rank-$k$ decomposition of spectral methods 
produce spoke-like patterns and can be identified using \spoken.  It is worth noting that \spoken is a
method to chip out large communities from graphs, and not necessarily attackers.  Verification of the \blatant lockstep behavior as fraudulent is required in this case.

\subsection{Blatant Attack/Camouflage}

Naturally, \blatant attacks with camouflage are more difficult to spot than
without.  Though the singular values of the attacks are above the threshold
$\sigma_k$ and the associated singular vectors appear in the rank-$k$
decomposition of spectral methods, Jiang et al. showed that rather than
axis-aligned spokes, the spectral subspace plots showed tilting rays.
\copycatch is also effective in detecting \blatant attacks with camouflage
(provided that the parameter choices are sufficiently
large to limit the rate of false positives), given that camouflage is ignored
in the case that an $m$, $n$, $p$ near-bipartite core is found for a subset of
$\mathcal{U}$ and $\mathcal{P}$ for a fixed snapshot of the input graph.

\subsection{Stealth Attack/No Camouflage}

As concluded in Section~\ref{sec:adversarial}, current detection schemes are
highly vulnerable to stealth attacks engineered to fall below
parameter thresholds of $\sigma_k$ for spectral methods or $m$, $n$, $p$ for
clustering methods.  To the best of our knowledge, no previous technique has
been able to successfully and effectively identify users involved in these
types of attacks.  Though stealth attacks may be individually of lesser
consequence to detect than larger cases of fraud, they have the insidious
property of being able to achieve the same number of fraudulent actions in a
more controlled and less detectable manner at the cost of simply creating more
fraud-related accounts.  In response to this threat, we propose the \method
algorithm for identifying such attacks in Section~\ref{sec:algo} and
demonstrate its effectiveness in Section~\ref{sec:exp}.  

\subsection{Stealth Attack/Camouflage}

Given that identifying small scale attacks has thus far been an open problem in
the context of fraud detection, the problem of identifying these with
camouflage has also gone unaddressed.  The difficulty in dealing with
camouflage is particularly apparent when considering user accounts with few
outgoing or incoming links, as is typically the case with smaller attacks.
From the perspective of a practitioner, it may appear that a truly fraudulent
account is mostly honest but with a few suspicious or uncharacteristic links
(insufficient to mark as fraudulent) or infrequently/unsavvily used due to the
small number of total links.  We demonstrate in Section~\ref{sec:exp} that
\method is robust to such smart attacks with moderate amounts of camouflage on
real social network data.

\section{Proposed Algorithm} 
\label{sec:algo}
Thus far, we have seen how existing state-of-the-art techniques have firm
effective detection thresholds and are entirely ineffective in detecting
stealth attacks that fall below this threshold.  Given this problem, it is
natural to consider the following question --- how can we identify the many
numerous small scale attacks that are prone to slipping below the radar of
existing techniques?  In this section, we formalize our problem definition and
propose \method as a suitable method for addressing this problem.

\subsection{Problem Formulation}

We identify the major problem to be addressed as follows:
\begin{problem}
{\bf Given} an input graph adjacency matrix {\bf A}, with rows and columns corresponding to users and objects (could be pages, articles, etc. or even other users), {\bf identify} stealth attacks which are undetectable given a desired decomposition rank-$k$ for {\bf A} (undetectable in that their singular values fall below the threshold $\sigma_k$).
\label{prob2}
\end{problem}
Note that Problem~\ref{prob2} is an exact foil to Problem~\ref{prob1}. In this paper, we primarily focus on smart attacks which fall below a practitioner-defined spectral threshold, given that a number of previous works mentioned have tackled the problem of discovering blatant attacks.   Given that this body of work is effective in detecting such attacks, we envision that the best means of boxing in attackers is a \emph{complementary} approach to existing methods, as our analysis in Section~\ref{sec:bigpic} is indicative of the lack of suitability of a one-size-fits-all technique for catching all attackers.  

\newcommand{\IND}[1][1]{\hspace{#1\algorithmicindent}}

\begin{algorithm}
\algsetup{linenosize=\tiny}
\scriptsize
\begin{algorithmic}[1]
\caption{\method algorithm pseudocode}
\label{alg:fbox}
\REQUIRE Input graph adjacency matrix {\bf A}, \\
\IND[1.6] Decomposition rank $k$, \\
\IND[1.6] Threshold $\tau$
\STATE userCulprits = $\left\{\right\}$
\STATE objectCulprits = $\left\{\right\}$
\STATE outDegrees = $rowSum({\bf A})$
\STATE inDegrees = $colSum({\bf A})$
\STATE $\lbrack{\bf U},{\bf \Sigma},{\bf V}\rbrack = svd({\bf A}, k)$
\FOR{each row $i$ in ${\bf U\Sigma}$}
\STATE recOutDegs = $\|{\bf (U\Sigma)_i}\|^2_2$
\ENDFOR
\FOR{each row $j$ in ${\bf V\Sigma}$}
\STATE recInDegs = $\|{\bf (V\Sigma)_j}\|^2_2$
\ENDFOR
\FOR{each unique $od$ in outDegrees}
\STATE nodeSet = $find(\textrm{outDegrees} == od)$
\STATE recOutDegSet = recOutDegs(nodeSet)
\STATE recThreshold = $percentile(\textrm{recOutDegSet},\tau)$
\FOR{each node $n$ in nodeSet}
\IF{recOutDegs($n$) $\leq$ recThreshold}
\STATE userCulprits = userCulprits + $n$
\ENDIF
\ENDFOR
\ENDFOR
\FOR{each unique $id$ in inDegrees}
\STATE nodeSet = $find(\textrm{inDegrees} == id)$ 
\STATE recInDegSet = recInDegs(nodeSet)
\STATE recThreshold = $percentile(\textrm{recInDegSet},\tau)$
\FOR{each node $n$ in nodeSet}
\IF{recInDegs($n$) $\leq$ recThreshold}
\STATE objectCulprits = objectCulprits + $n$
\ENDIF
\ENDFOR
\ENDFOR
\RETURN userCulprits, \\
\IND[2.2] objectCulprits
\end{algorithmic}
\end{algorithm}

\subsection{Description}

As per the problem formulation, we seek to develop a solely graph-based method,
which will be able to complement existing fraud detection techniques by
discerning previously undetectable attacks.  In Section~\ref{sec:adversarial},
we demonstrated that smaller attacks are particularly characterized by
comparatively low singular values (below $\sigma_k$), and thus do not appear in
the singular vectors given by a rank $k$ decomposition.  Assuming an isolated
attack which has been engineered to fall below the detection threshold, the
users/objects comprising the attack will have absolutely \emph{no} projection
onto any of the top-$k$ left and right singular vectors respectively.  In the
presence of camouflage, projection of the culprit nodes may increase slightly
given some nonzero values in the corresponding indices in one or more of the
vectors.  In either case, we note that nodes involved in these attacks have the
unique property of having zero or almost-zero projections in the projected
space.  Given this observation, two questions naturally arise: (a) how can we
effectively capture the extent of projection of a user or object? and (b) is
there a pattern to how users or objects project into low-rank subspaces?

In fact, we can address the first question by taking advantage of the norm-preserving property of SVD, which states that the row vectors of a full rank decomposition and associated projection will retain the same $l_2$ norm or vector length as in the original space.   That is, for $k = rank({\bf A})$,
\begin{displaymath}
\|{\bf A_i}\|_2 = \|{\bf \textrm{(}U\Sigma\textrm{)}_i}\|_2 \textrm{ for } i \leq u
\end{displaymath}

In the same fashion, one can apply the norm-preserving property to decomposition of ${\bf A^T}$ to show

\begin{displaymath}
\|{\bf {A^T}_j}\|_2 = \|{\bf \textrm{(}V\Sigma\textrm{)}_j}\|_2 \textrm{ for } j \leq o
\end{displaymath}

\begin{figure}[!t]
\centering
\begin{subfigure}{.24\textwidth}
  \centering
  \includegraphics[width=\linewidth]{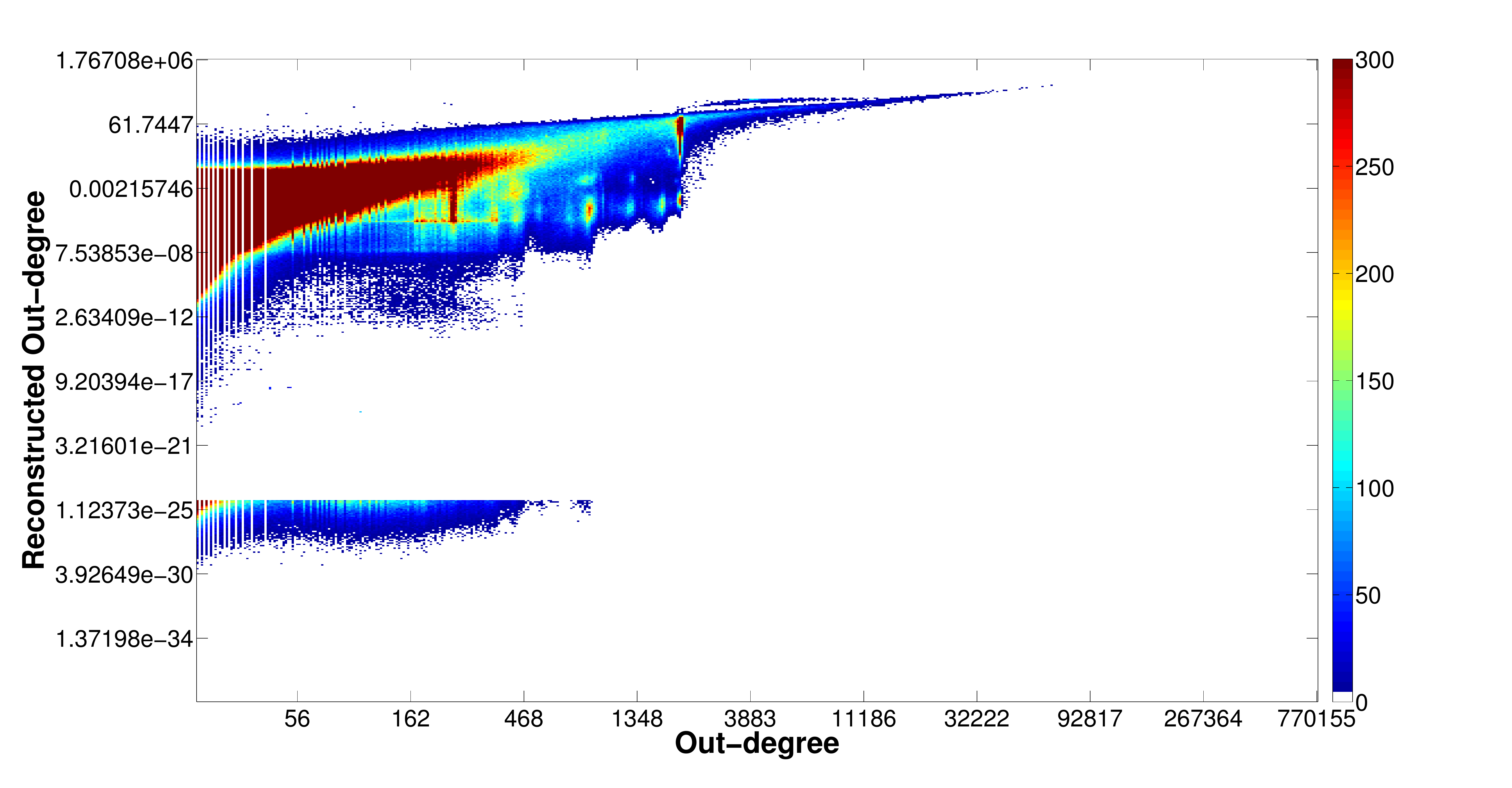}
  \caption{OSRM for Twitter fans}
  \label{fig:osrm}
\end{subfigure}%
\begin{subfigure}{.24\textwidth}
  \centering
  \includegraphics[width=\linewidth]{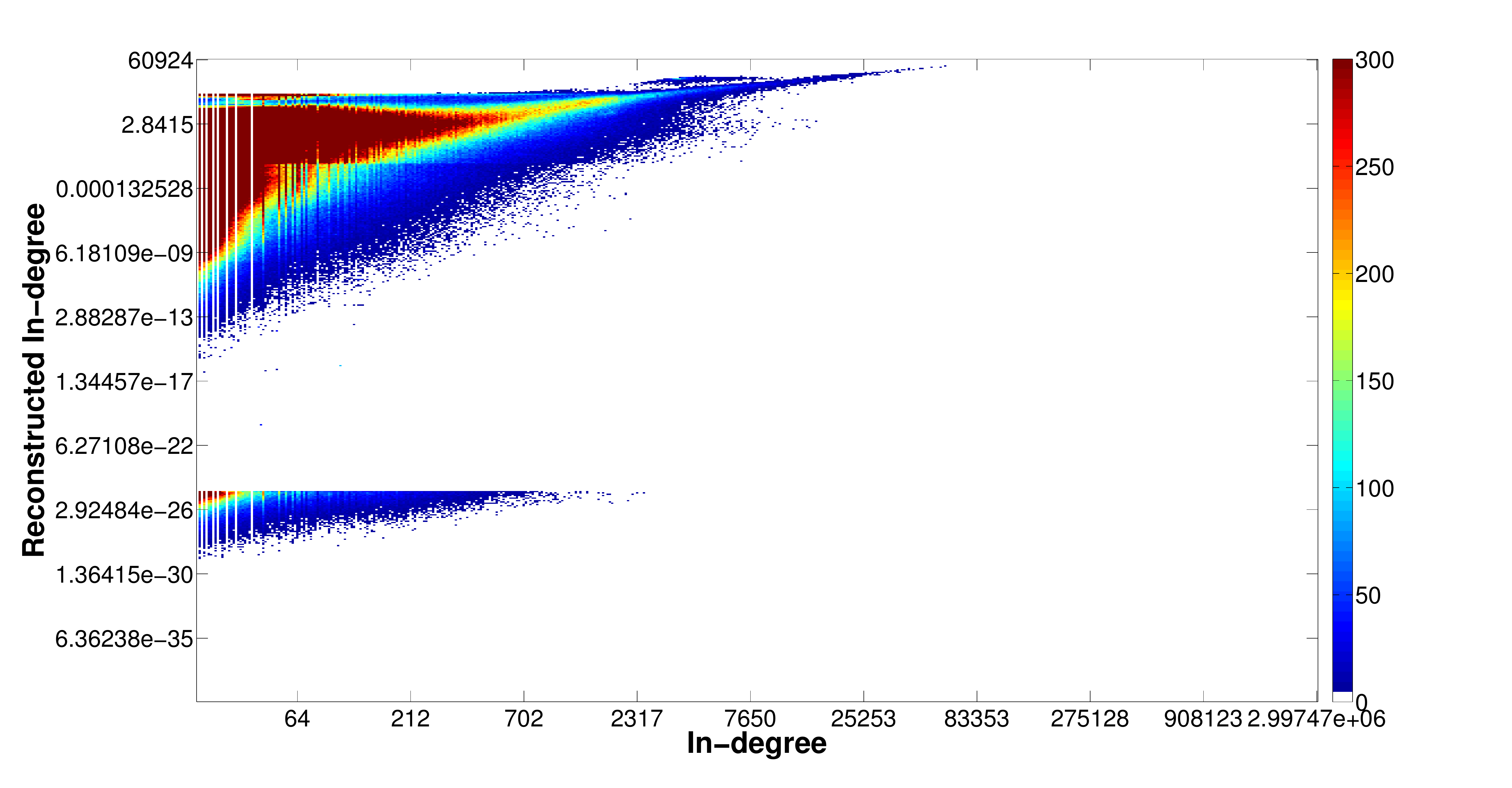}
  \caption{ISRM for Twitter idols}
  \label{fig:isrm}
\end{subfigure}
\caption{SRMs show correlation between the reconstruction degree and suspiciousness of nodes.  (a) and (b) show the SRMs produced from analysis on the Twitter social graph.}
\label{fig:srmtwitter}
\end{figure}

Since the $l_2$ norms are preserved in a full rank decomposition, it is obvious that the sum of squares of components are also preserved.  Note that for the 0-1 adjacency matrix {\bf A} we consider here, the sum of squares of components of the $i$th row vector corresponds to the out-degree of user $i$ and the sum of squares of components of the $j$th column vector corresponds to the in-degree of object $j$ --- given these considerations, we define the degree of a node in a given subspace as the squared $l_2$ norm of its link vector in that subspace.  Thus, for a full rank decomposition, the out-degree given by $\|{\bf A_i}\|^2_2$ and \emph{reconstructed} out-degree given by $\|{\bf \textrm{(}U\Sigma\textrm{)}_i}\|^2_2$ of user $i$ are equal.  The same can be said for the in-degree and \emph{reconstructed} in-degree of object $j$.  For rank $k$ decompositions where $k < rank({\bf A})$ (guaranteed in practical use of spectral methods), it is clear that the true degrees upper bound the reconstructed degrees.  Thus, we can capture the extent of projection of a user by the tuple of his true out-degree and reconstructed out-degree, and we can capture the extent of projection of an object by the tuple of its true in-degree and reconstructed in-degree.

We conjecture that due to the different graph connectivity patterns of dishonest and honest users as well as dishonest and honest objects, their projections in terms of reconstructed degrees should also vary.  Intuitively, dishonest users who either form isolated components or link to dishonest objects will project poorly and have characteristically low reconstruction degrees, whereas honest users who are well-connected to real products and brands should project more strongly and have characteristically higher reconstruction degrees.  In fact, we find that in real data, users and objects have certain ranges in which they commonly reconstruct in the projected space.  Figure~\ref{fig:srmtwitter} shows the OSRM (\emph{Out-link Spectral Reconstruction Map}) and ISRM (\emph{In-link Spectral Reconstruction Map}) for a large, multi-million node and multi-billion edge social graph from Twitter, where we model follower (fan) and followee (idol) behavior.  The data is represented in heatmap form to indicate the distribution of reconstructed degrees for each true degree.  The SRMs indicate that for each true degree, there is a tailed distribution with most nodes reconstructing in a common range and few nodes reconstructing as we move away from this range in either direction.  Most notably, there are a large number of nodes with degrees up to the hundreds with an almost-zero reconstruction, depicted by a well separated point cloud at the bottom of both SRMs.  For higher true degree values in the thousands, nodes are more sparse and rarely project as poorly as for lower true degrees, but many points at these degree values reconstruct several degrees of magnitude lower than the rest.  These observations serve to substantiate our conjecture that poorly reconstructing nodes are suspicious, but what about the  well reconstructing nodes?  Interestingly, we find that nodes which reconstruct on the high range of the spectrum for a given degree have many links to popular (and commonly Twitter-verified) accounts.  We do not classify such behavior as suspicious in the OSRM context, as it is common for Twitter users to follow popular actors, musicians, brands, etc.  We do not classify such behavior as suspicious in the ISRM context either, as popular figures tend to more commonly be followed by other popular figures.  At the bottom of the reconstruction spectrum, however, we most commonly find accounts which demonstrate a number of notably suspicious behaviors in the context of their followers/followees and the content of their Tweets --- more details are given in Section~\ref{sec:exp}.



Based on our intuitive conjecture and empirical verification, we focus our \method algorithm on identifying nodes with characteristically poor reconstructed degree in comparison to other nodes of the same true degree as suspicious.  Specifically, we mark the bottom $\tau$ percent of nodes per fixed degree for both users and objects as culprit nodes.  We outline the high-level steps of \method in Algorithm~\ref{alg:fbox}.


\section{Experiments}
\label{sec:exp}

\subsection{Datasets} 
\label{sub:Datasets}

For our experiments we primarily use two datasets: the who-follows-whom Twitter graph and the who-rates-what Amazon graph.  The Twitter graph was scraped
by Kwak et al. in 2010 and contains 41.7 million users with 1.5 billion edges~\cite{kwak10www}.  We showed the distribution of singular values in Figure
\ref{fig:twitter_svd}.  The Amazon ratings graph was scraped in March 2013 by
McAuley and Leskovec \cite{mcauley2013hidden} and contains 29 million reviews
from 6 million users about 2 million products.  The distribution of singular
values can be seen in Figure \ref{fig:amazon}.  Our analysis is conducted both directly and via synthetic attacks.  





\subsection{\method on real Twitter accounts} 

\begin{figure}[!t]
\centering
\begin{subfigure}{.24\textwidth}
  \centering
  \includegraphics[width=0.98\textwidth]{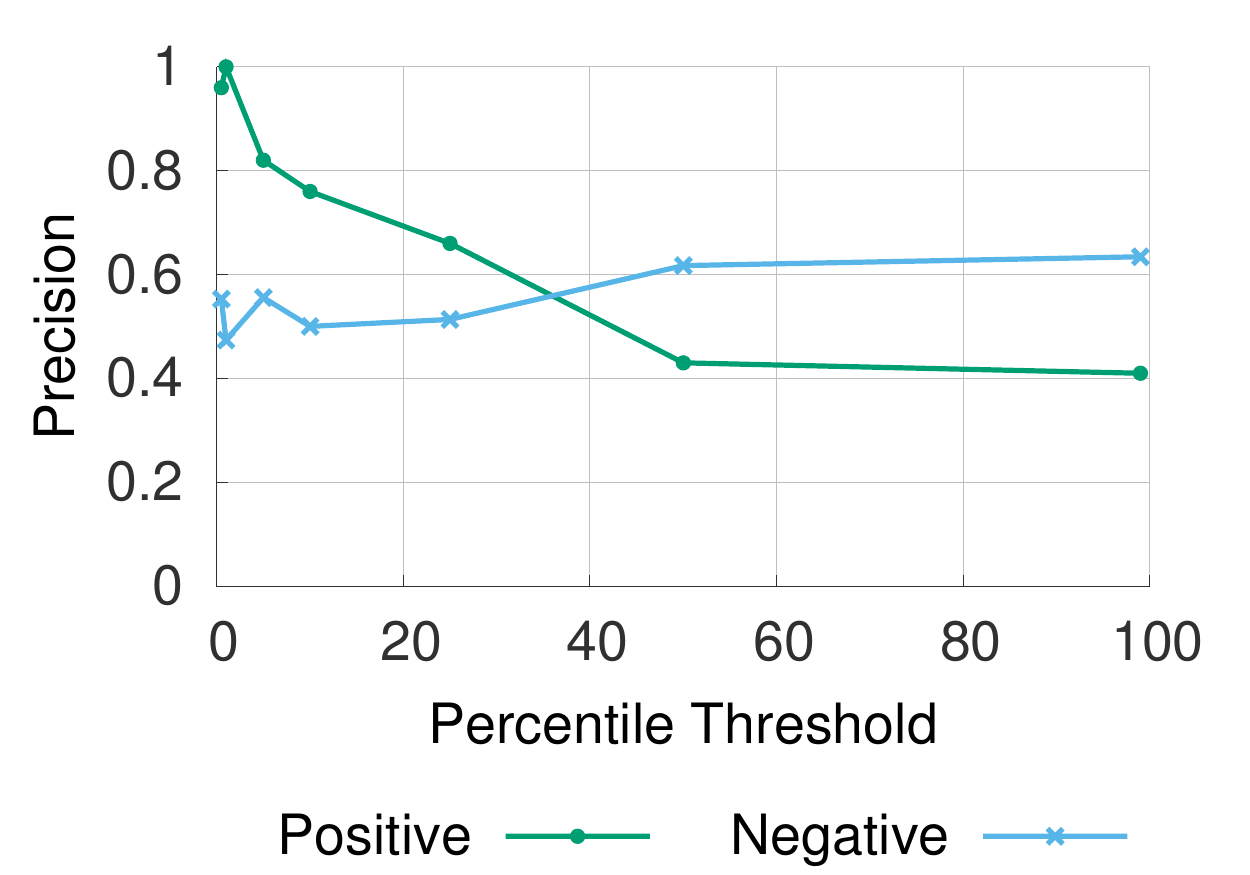}
  \caption{Fans}
  \label{fig:twitter_fans}
\end{subfigure}%
\begin{subfigure}{.24\textwidth}
  \centering
  \includegraphics[width=0.98\textwidth]{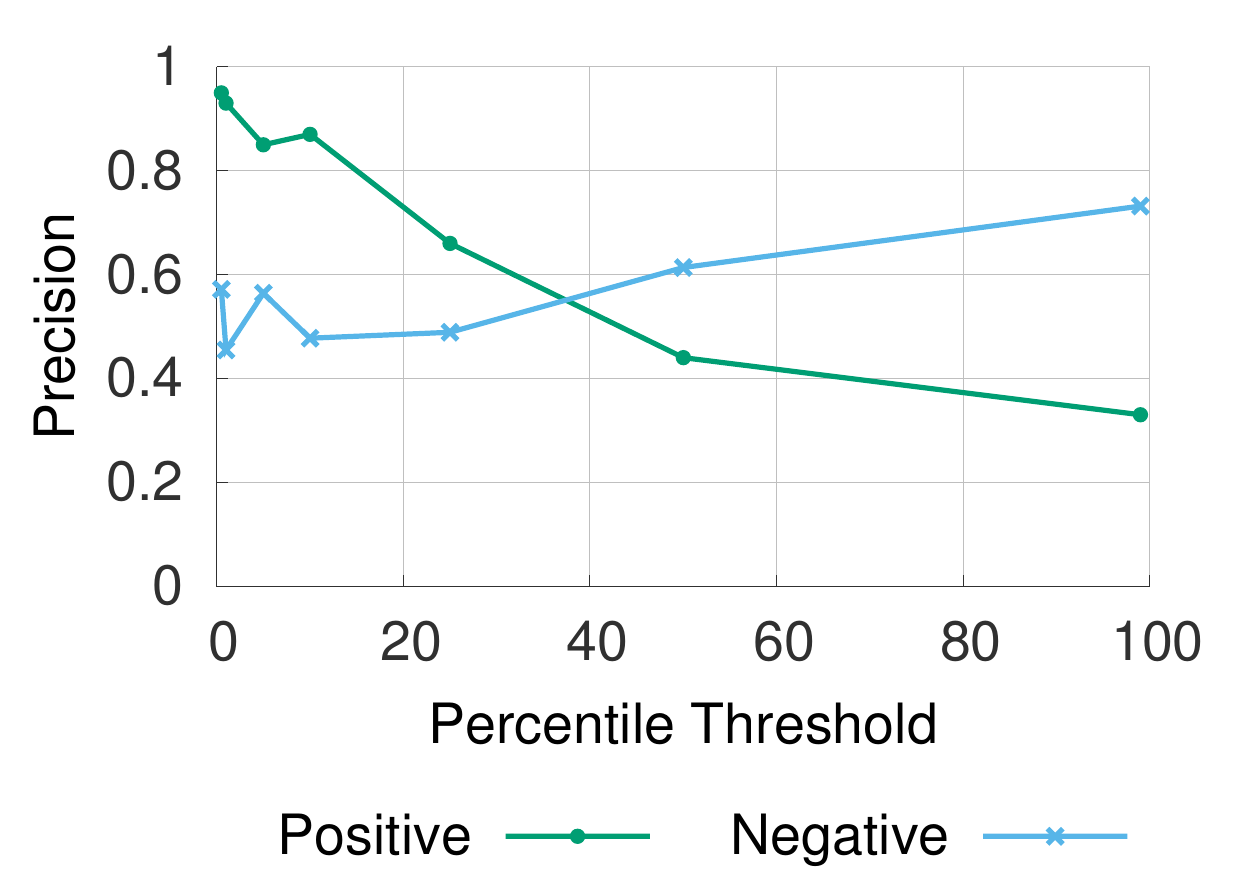}
  \caption{Idols}
  \label{fig:twitter_idols}
\end{subfigure}%
\caption{(a) and (b) show \method's strong predictive value and low false-discovery rate in identifying suspicious accounts.}
\label{fig:twitter_results}
\end{figure}

To show our effectiveness in catching smart link fraud attacks on real data, we
conducted a classification experiment on data from the Twitter graph.
Specifically, we collected the culprit results for suspicious fans and idols with degree at least 20 (to avoid catching unused accounts) for seven different values of the detection threshold $\tau$, at 0.5, 1, 5, 10,
25, 50 and 99 percentile.  For each combination of $\tau$ value and user type
(fan or idol), we randomly sampled 50 accounts from the ``culprit-set'' of
accounts classified as suspicious by \method and another 50 accounts from the
remainder of the graph in a 1:1 fashion, for a total of 1400 accounts.  We randomly organized
and labeled these accounts as suspicious or honest (ignoring foreign and protected accounts) based on several criteria
--- particularly, we identified suspicious behavior as accounts with some
combination of the following characteristics:

\begin{compactitem}
\item Suspension by Twitter since data collection
\item Spammy or malicious tweets (links to adware/malware)
\item Suspicious username, or followers/followees have suspicious usernames (with common prefixes/suffixes)
\item Very few tweets (\textless5) but numerous (\textgreater20) followees who are themselves suspicious
\item Sparse profile but numerous (\textgreater20) followees who are themselves suspicious 
\end{compactitem} 


Figure \ref{fig:twitter_results}
shows how the performance of \method varies with the threshold $\tau$ for
Twitter fans and idols.  As evidenced by the results, \method is able to
correctly discern suspicious accounts with 0.93+ precision for $\tau \leq 1$
for both fans and idols.
And as expected, increasing $\tau$  results
in lower precision.  
As with many informational retrieval and spam detection problems, there are an
unbounded number of false negatives, making recall effectively impossible to
calculate.
Rather, we use the negative precision and observe that it increases as we increase
$\tau$.  Ultimately, because \method is meant to be a complementary method to
catch new cases of fraud, we do not believe that missing some of the attackers
already caught by other methods is a major concern.
With these considerations, we recommend conservative
threshold values for practitioner use.  On Twitter data, we found roughly 150
thousand accounts classified as suspicious between the SRMs for $\tau
= 1$.

\begin{figure*}[htbp!]
\centering
\begin{subfigure}{.3\textwidth}
  \centering
  \includegraphics[width=0.98\textwidth]{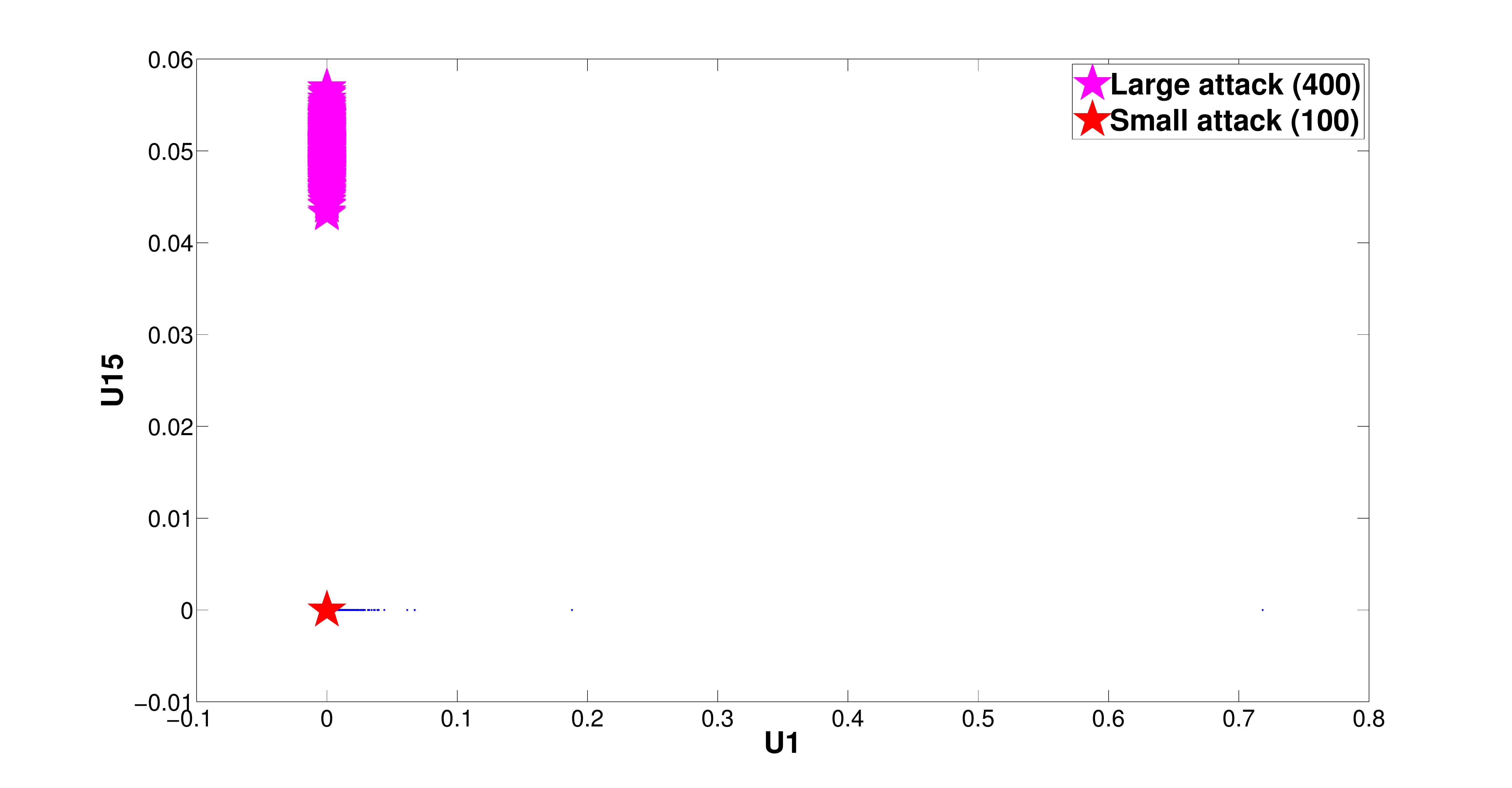}
  \vspace{9mm}
  \caption{Spectral subspace plot}
  \label{fig:eigenplot_injections}
\end{subfigure}%
\begin{subfigure}{.3\textwidth}
  \centering
  \includegraphics[width=0.98\textwidth]{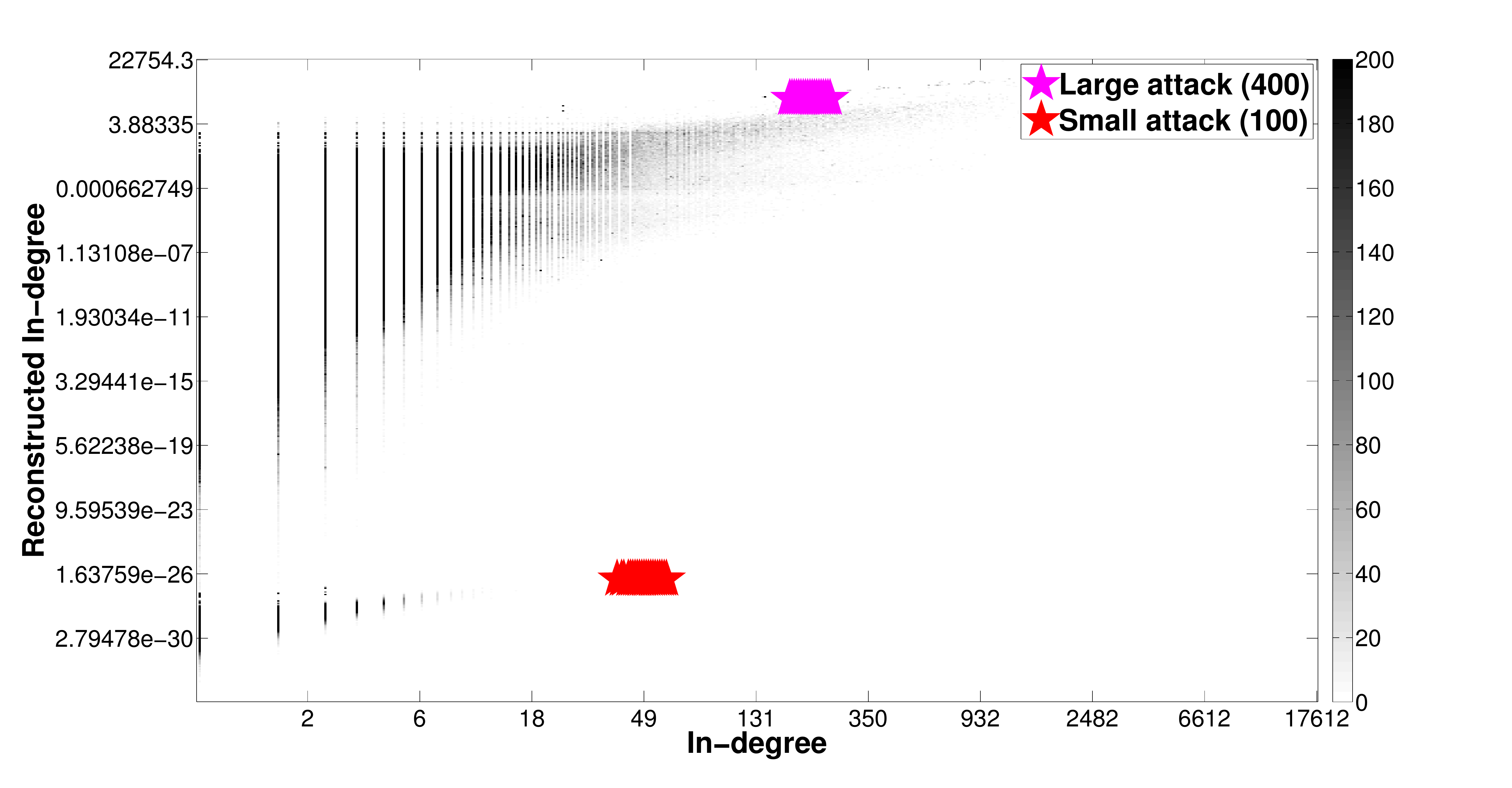}
  \vspace{9mm}
  \caption{ISRM}
  \label{fig:srm_injections}
\end{subfigure}%
\begin{subfigure}{.3\textwidth}
  \centering
  \includegraphics[width=0.98\textwidth]{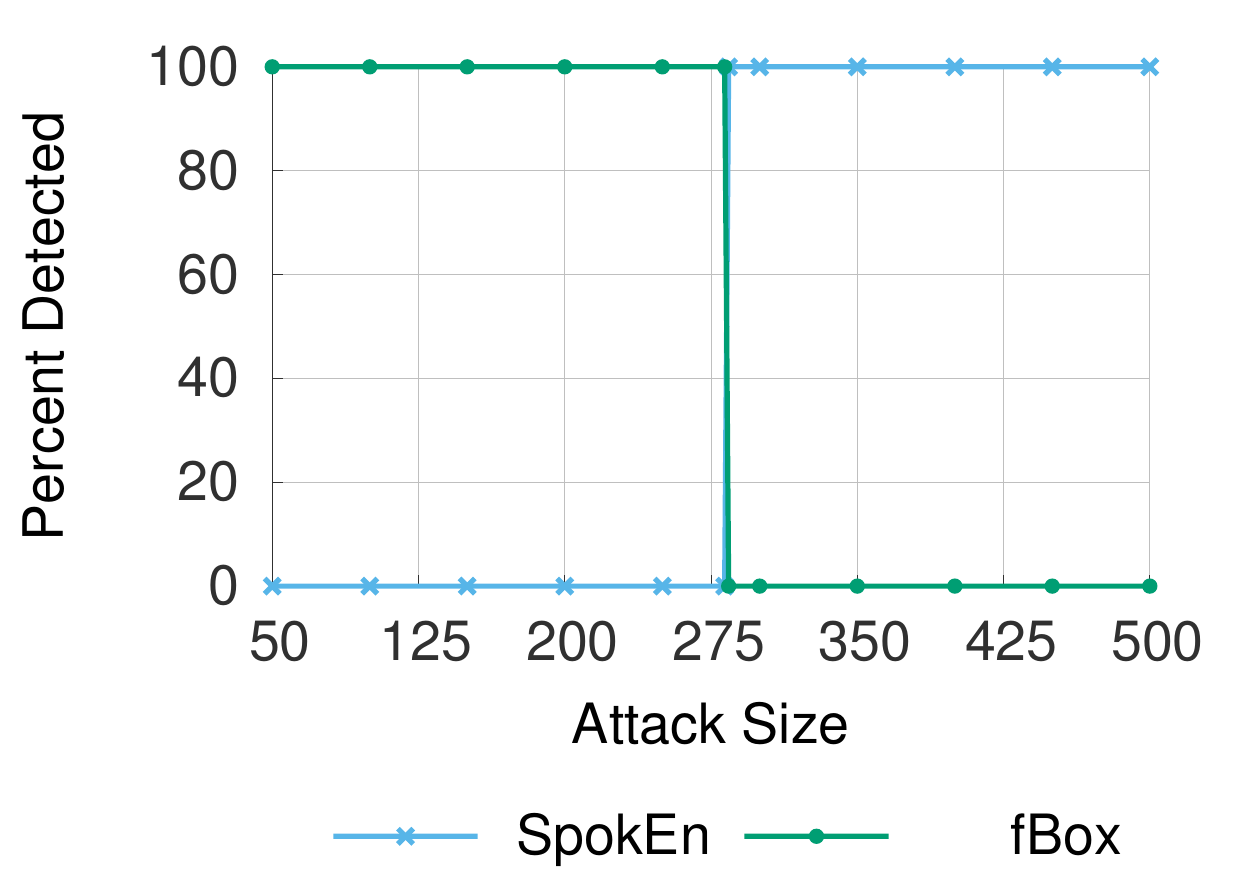}
  \caption{Method complementarity}
  \label{fig:method_comparison}
\end{subfigure}%
\caption{\method and \spoken are complementary, with \method detecting smaller stealth attacks missed by \spoken.  (a) shows how spectral subspace plots identify \blatant attacks but ignore smaller ones. (b) shows the ISRM
plot for the same injections, clearly identifying the suspiciousness of the smart attack.  (c) depicts the complementary nature of \method and spectral methods in detecting attacks at various scales.}
\label{fig:injection_tests}
\end{figure*}

\begin{figure*}[htbp!]
\centering
\begin{subfigure}{.3\textwidth}
  \centering
  \includegraphics[width=0.98\textwidth]{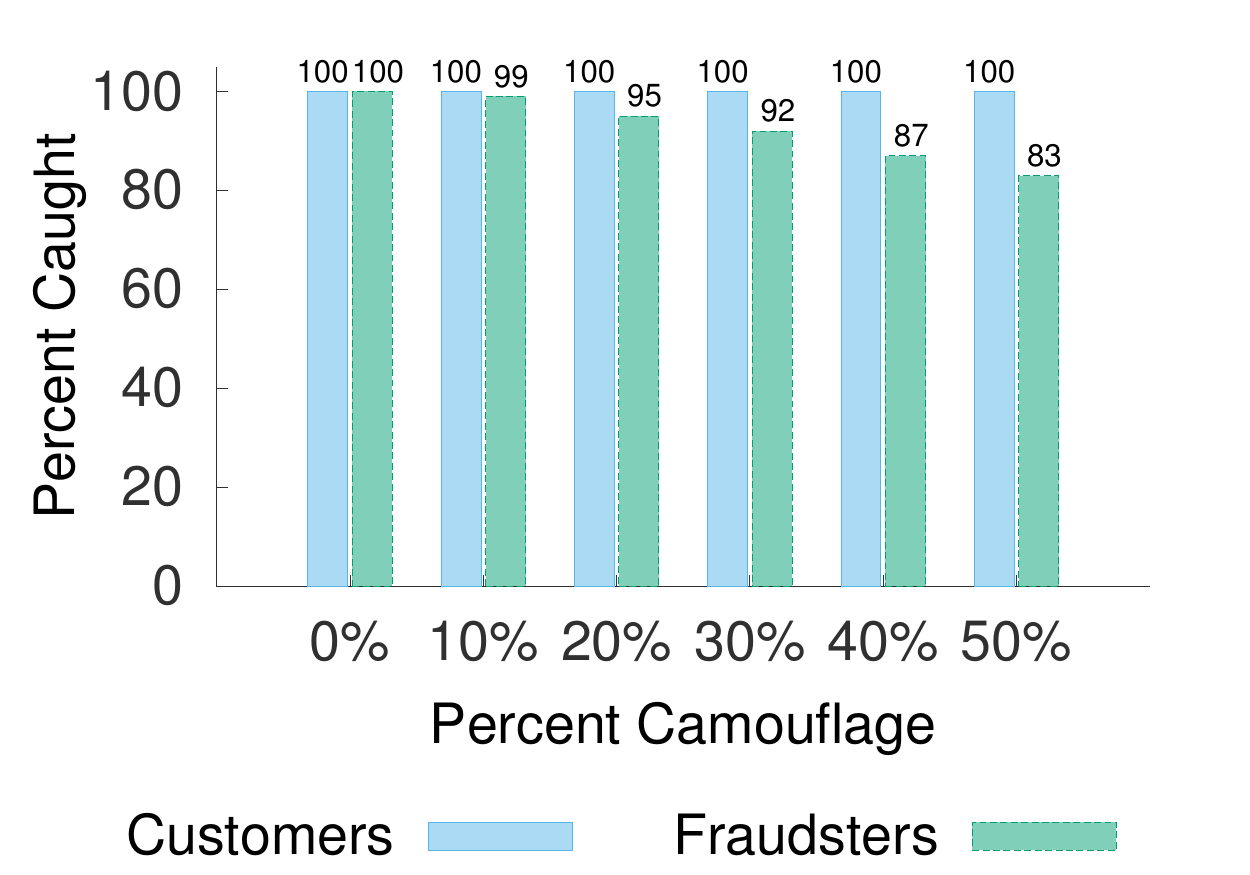}
  \caption{$100 \times 100$ attack}
  \label{fig:camo_100}
\end{subfigure}%
\begin{subfigure}{.3\textwidth}
  \centering
  \includegraphics[width=0.98\textwidth]{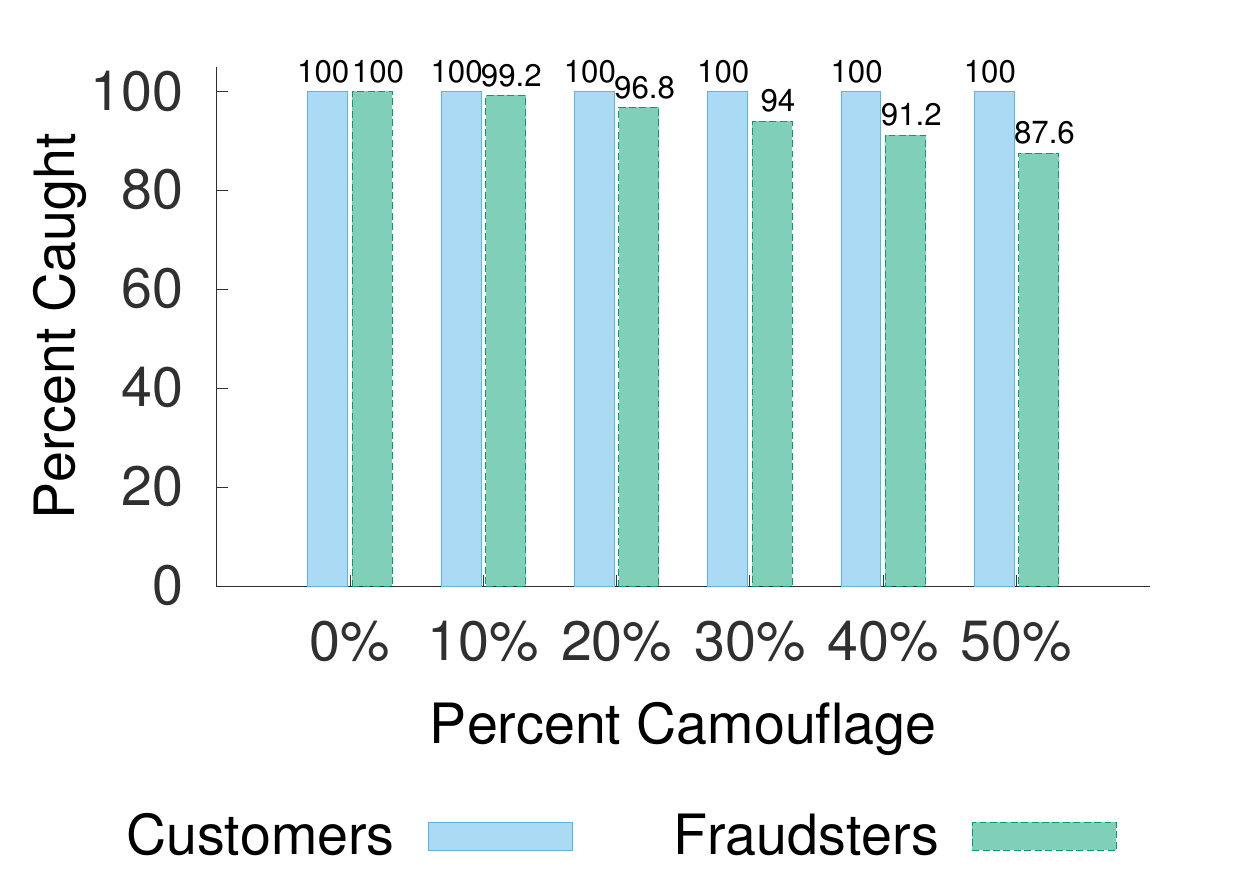}
  \caption{$250 \times 250$ attack}
  \label{fig:camo_250}
\end{subfigure}%
\begin{subfigure}{.3\textwidth}
  \centering
  \includegraphics[width=0.98\textwidth]{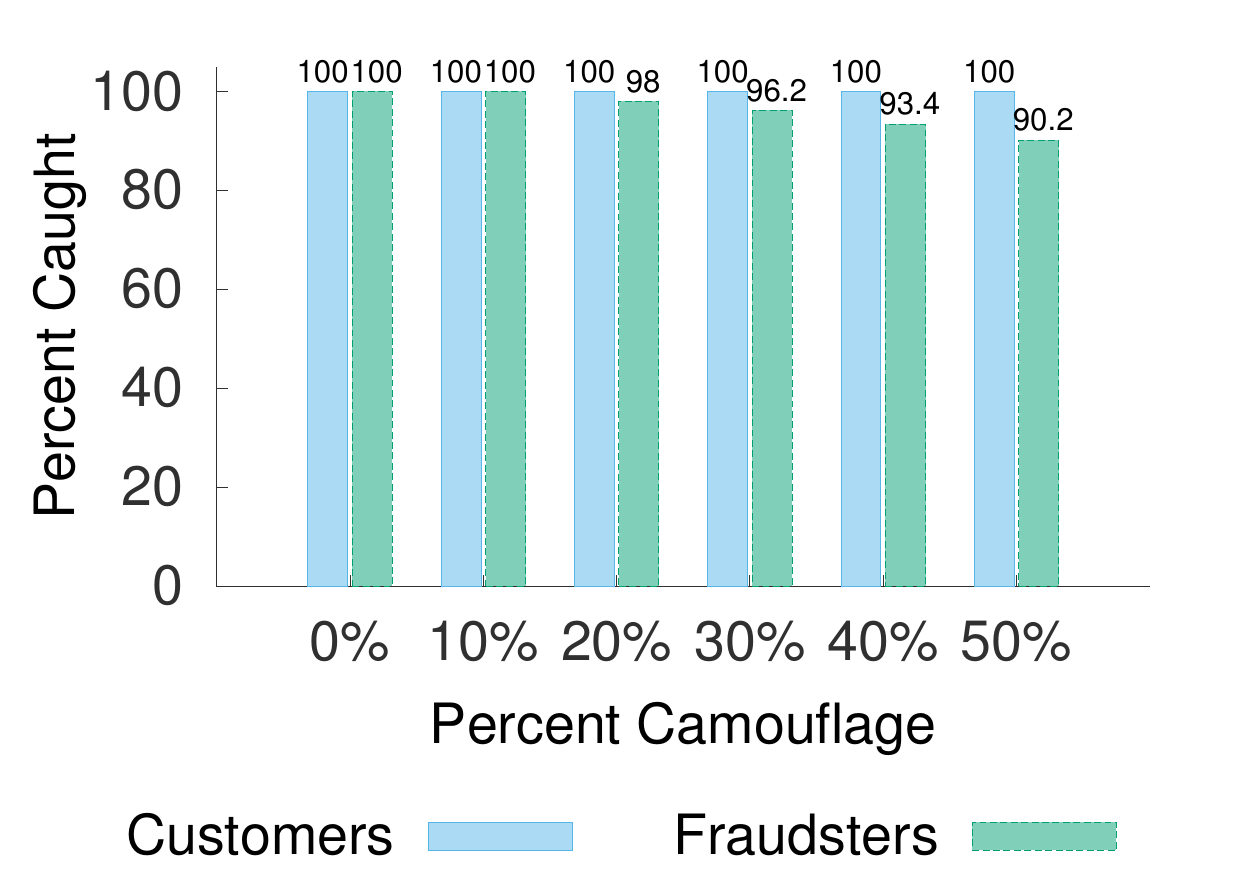}
  \caption{$500 \times 500$ attack}
  \label{fig:camo_500}
\end{subfigure}
\caption{(a), (b) and (c) show \method's robustness to moderate amounts of camouflage for attack sizes of 100, 250 and 500.}
\label{fig:camo_results}
\end{figure*}

\subsection{Complementarity of \method} 
As mentioned before, \method is complementary to spectral
techniques and is effective in catching smart attacks that adversaries could engineer to avoid detection by these techniques.  We demonstrate this claim using both synthetically formulated attacks on the Amazon network as well as comparing the performance of both \method and \spoken on the Twitter network.  In the first experiment, we inject random attacks of scale 100 ($100 \times 100$) and 400 ($400 \times 400$), each with density $p = 0.5$ into the Amazon graph and compare the effectiveness of spectral subspace plots and SRMs in spotting these attacks.  Figure
\ref{fig:eigenplot_injections} shows the spectral subspace plot for the 1st and 15th
components of the SVD, corresponding to one naturally existing community and the \blatant attack, respectively.  The plot clearly shows nodes involved in the \blatant attack as a spoke
pattern, but groups the nodes involved in the small attack along with
many honest nodes that reconstruct poorly in these components at the origin point.  However, in Figure \ref{fig:srm_injections}, we see that the smaller injection is identified as clearly suspicious with distinct separation from other legitimate behavior.

We additionally tested both \method and \spoken on a number of
injections sizes, each random attacks with $p = 0.5$.  Figure \ref{fig:method_comparison} shows the fraction of the
attacking fans caught by each algorithm.  As seen in the figure, the two
methods are clearly complementary, with \method catching all attacks
that \spoken misses.  This verifies the analysis in
Section \ref{sec:adversarial} and substantiates \method's suitability for catching stealth attacks that produce leading singular value $\sigma' < \sigma_k$.


In our second experiment, we compared the performance of both \method and \spoken on a sample of 65743 accounts selected from the Twitter graph.  For each of these accounts, we queried the Twitter API to collect information regarding whether the account was suspended or had posted Tweets promoting adware/malware (checked via Google SafeBrowsing), and if so we marked the account as fraudulent.  This ground truth marking allows us to unbiasedly measure the complementarity of \method and \spoken in catching users that are surely malicious.  Of these users, 4025 were marked as fraudulent via Twitter (3959) and Google SafeBrowsing (66).  For rank $k = 50$, \spoken produced 8211 suspicious accounts whereas \method (with $\tau = 1$) produced 149899.  The user sets identified by both methods were found to be completely distinct, suggesting that the methods are indeed complementary.  Furthermore, \method identified 1133 suspicious accounts from the sampled dataset, of which only 347 were caught via Twitter and Google SafeBrowsing, suggesting that roughly 70\% of \method-classified suspicious accounts are missed by Twitter.


\subsection{\method in the face of camouflage} 
One key point in dealing with intelligent attackers is ensuring that \method
is robust in detecting attacks with moderate amounts of camouflage.  To measure
our performance in such a setting,
we ran \method on a variety of attack sizes in our target range 
and for each attack
varied the amount of camouflage added.  
In our model, we include 
camouflage by following honest accounts at random.  For a random attack of
size $n \times n$ and edge probability $p$, we vary the percent of idols of
fraudulent fans that are camouflage: for 0\% camouflage
each fan follows the $pn$ customers only and for 50\% camouflage
each attacker node follows $pn$ customers and $pn$ random honest idols --- in
general, the percent of camouflage $r$ for $g$ camouflage links is defined as
$\frac{100g}{g+pn}$.
We ran this test for attacks of size 100, 250, and 500 (all below the $\sigma_{25} = 1143.4$ detection threshold) with $p = 0.5$ on the
Twitter graph.

Figure \ref{fig:camo_results} demonstrates \method's robustness --- for all configurations of attack size and camouflage, we catch {\em all} customer idols and over 80\% of
fraudulent fans.  As attack size increases, increased
camouflage is less impactful (intuitively, larger attacks are more flagrant), with \method catching over 90\% of the
fraudulent fans even with 50\% camouflage.

Analysis on \emph{fame}, where customers buying links also have honest links was not conducted.  Customer fame is the analog of attacker camouflage.  However, we expect similar results in detection of accounts in the presence of fame given the symmetry of SVD and \method's disjoint user/object reconstruction.  However, the presence of fame is less realistic in many applications --- for example, in the Twitter context, it is difficult for a spammy account to get honest fans whereas fraudulent fans can follow real idols at will.



\subsection{Scalability of \method}

The running time of \method is dominated by the (linear) large matrix-vector
multiplication per iteration of the Lanczos algorithm to compute SVD for large,
sparse matrices.  Figure~\ref{fig:scalennz} depicts the linear runtime of \method for $k = 25$ while varying number of non-zeros.
\begin{figure}
\centering
\includegraphics[width=0.38\textwidth]{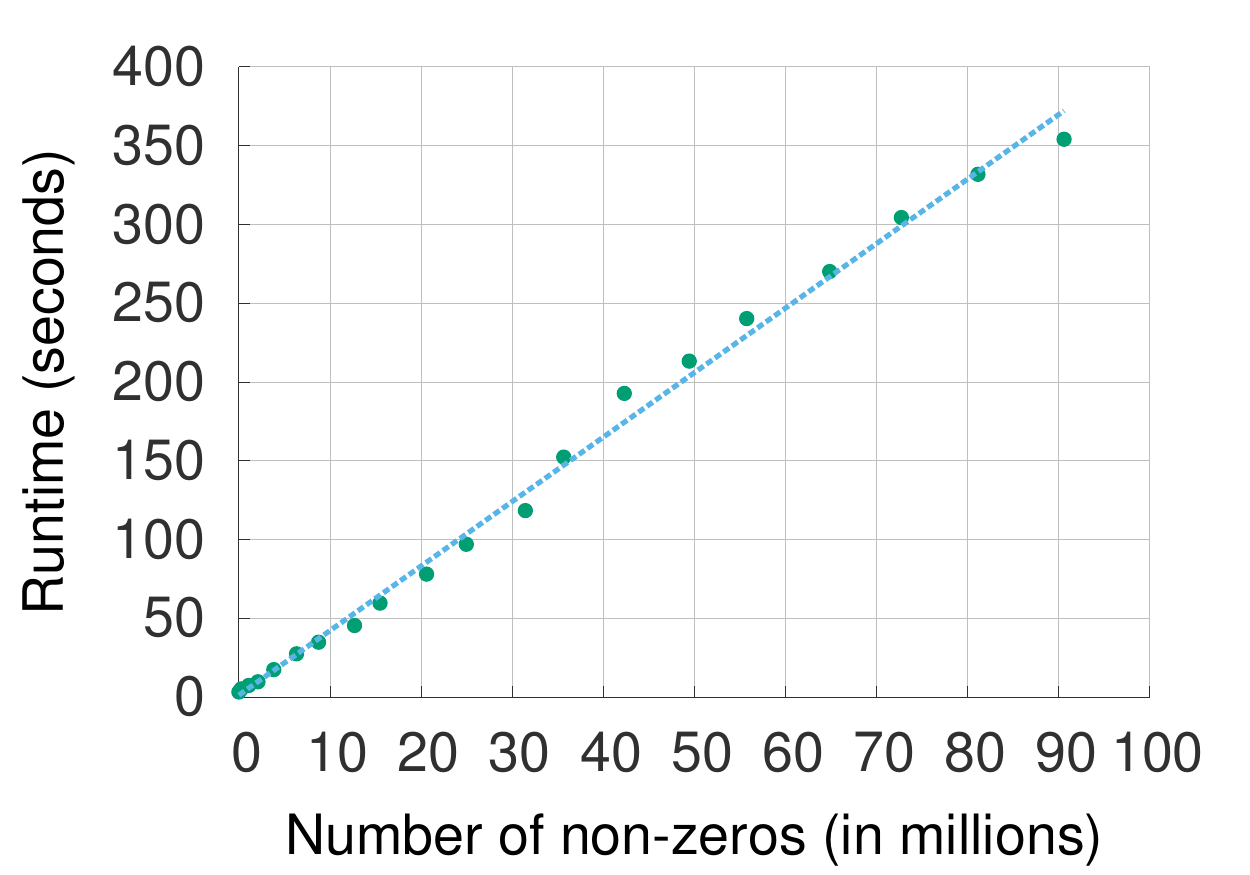}
\caption{\method scales linearly on the size of input data.}
\label{fig:scalennz}
\end{figure}

\section{Conclusions}
\label{sec:concl}
In this work, we approached the problem of distinguishing dishonest attackers and their customers from honest users in the context of online social network or web-service graphs using a graph-based approach (using the adjacency matrix representing user/object relationships).  Our main contributions are:

\begin{compactenum}
 \item {\bf Theoretical analysis:} We examine several state-of-the-art fraud detection methods from an adversarial point-of-view and provide theoretical results pertaining to the susceptibility of these methods to various types of attacks.
 \item {\bf \method~algorithm:} We detail \method, a method motivated by addressing the blind-spots discovered in theoretical analysis, for detecting a class of \emph{stealth} attacks which previous methods are effectively unable to detect.
 \item {\bf Effectiveness on real data:} We apply \method to a large Twitter who-follows-whom dataset from 2010 and discover many tens of thousands of suspicious users whose accounts remain active to date
\end{compactenum}

Our experiments show that our method is scalable, effective in detecting a complementary range of attacks to existing methods and robust to a reasonable degree of camouflage for small and moderately sized stealth attacks.

\bibliographystyle{abbrv}
\bibliography{BIB/christosref,BIB/other}

\begin{thebibliography}{10}

\bibitem{akoglu2010oddball}
L.~Akoglu, M.~McGlohon, and C.~Faloutsos.
\newblock Oddball: Spotting anomalies in weighted graphs.
\newblock {\em Advances in Knowledge Discovery and Data Mining}, pages
  410--421, 2010.

\bibitem{cobafi}
A.~Beutel, K.~Murray, C.~Faloutsos, and A.~J. Smola.
\newblock Cobafi - collaborative bayesian filtering.
\newblock In {\em WWW}, 2014.

\bibitem{copycatch13}
A.~Beutel, W.~Xu, V.~Guruswami, C.~Palow, and C.~Faloutsos.
\newblock Copycatch: stopping group attacks by spotting lockstep behavior in
  social networks.
\newblock In {\em WWW}, pages 119--130. ACM, 2013.

\bibitem{domingoskdd04}
N.~Dalvi, P.~Domingos, S.~Sanghai, D.~Verma, et~al.
\newblock Adversarial classification.
\newblock In {\em SIGKDD}, pages 99--108. ACM, 2004.

\bibitem{erdos1959}
P.~Erd{\H{o}}s and A.~R{\'e}nyi.
\newblock On random graphs.
\newblock {\em Publicationes Mathematicae Debrecen}, 6:290--297, 1959.

\bibitem{ghosh2012understanding}
S.~Ghosh, B.~Viswanath, F.~Kooti, N.~K. Sharma, G.~Korlam, F.~Benevenuto,
  N.~Ganguly, and K.~P. Gummadi.
\newblock Understanding and combating link farming in the twitter social
  network.
\newblock In {\em WWW}, pages 61--70. ACM, 2012.

\bibitem{berkeleyspam}
C.~Grier, K.~Thomas, V.~Paxson, and M.~Zhang.
\newblock @ spam: the underground on 140 characters or less.
\newblock In {\em CCS}. ACM, 2010.

\bibitem{huang2008spectral}
L.~Huang, D.~Yan, M.~I. Jordan, and N.~Taft.
\newblock Spectral clustering with perturbed data.
\newblock In {\em NIPS}, volume~21, 2008.

\bibitem{jiang14}
M.~Jiang, P.~Cui, A.~Beutel, C.~Faloutsos, and S.~Yang.
\newblock Inferring strange behavior from connectivity pattern in social
  networks.
\newblock In {\em PAKDD}, 2014.

\bibitem{kang2014heigen}
U.~Kang, B.~Meeder, E.~E. Papalexakis, and C.~Faloutsos.
\newblock Heigen: Spectral analysis for billion-scale graphs.
\newblock {\em TKDE}, 26(2):350--362, 2014.

\bibitem{kwak10www}
H.~Kwak, C.~Lee, H.~Park, and S.~Moon.
\newblock {W}hat is {T}witter, a social network or a news media?
\newblock In {\em WWW}, pages 591--600. ACM, 2010.

\bibitem{leskovec2010signed}
J.~Leskovec, D.~Huttenlocher, and J.~Kleinberg.
\newblock Signed networks in social media.
\newblock In {\em SIGCHI}, pages 1361--1370. ACM, 2010.

\bibitem{lowd2005adversarial}
D.~Lowd and C.~Meek.
\newblock Adversarial learning.
\newblock In {\em SIGKDD}, pages 641--647. ACM, 2005.

\bibitem{malspot14}
C.-H. Mao, C.-J. Wu, E.~E. Papalexakis, C.~Faloutsos, and T.-C. Kao.
\newblock Malspot: {Multi$^2$} malicious network behavior patterns analysis.
\newblock In {\em PAKDD}, 2014.

\bibitem{mcauley2013hidden}
J.~McAuley and J.~Leskovec.
\newblock Hidden factors and hidden topics: understanding rating dimensions
  with review text.
\newblock In {\em RecSys}, pages 165--172. ACM, 2013.

\bibitem{netflix}
Netflix.
\newblock Netflix competition.
\newblock 2006.

\bibitem{ng2001spectralA}
A.~Y. Ng, M.~I. Jordan, and Y.~Weiss.
\newblock On spectral clustering: Analysis and an algorithm.
\newblock {\em NIPS}, 14:849--856, 2001.

\bibitem{netprobe07pandit}
S.~Pandit, D.~H. Chau, S.~Wang, and C.~Faloutsos.
\newblock Netprobe: a fast and scalable system for fraud detection in online
  auction networks.
\newblock In {\em WWW}, pages 201--210. ACM, 2007.

\bibitem{eigenspokes10}
B.~Prakash, M.~Seshadri, A.~Sridharan, S.~Machiraju, and C.~Faloutsos.
\newblock Eigenspokes: Surprising patterns and community structure in large
  graphs.
\newblock {\em PAKDD}, 84, 2010.

\bibitem{shrivastava2008mining}
N.~Shrivastava, A.~Majumder, and R.~Rastogi.
\newblock Mining (social) network graphs to detect random link attacks.
\newblock In {\em ICDE}. IEEE, 2008.

\bibitem{ying2011spectrum}
X.~Ying, X.~Wu, and D.~Barbar{\'a}.
\newblock Spectrum based fraud detection in social networks.
\newblock In {\em ICDE}, pages 912--923. IEEE, 2011.

\end{thebibliography}

\end{document}